\documentclass[11pt]{article}

\usepackage[show]{ed}

\usepackage{graphicx}

\usepackage{geometry}
\usepackage[utf8]{inputenc}
\usepackage{xifthen}
\usepackage{xspace} 
\usepackage[sans]{dsfont}
\usepackage{amsmath}
\usepackage{amssymb}
\usepackage[all]{xy}
\xyoption{v2}
\usepackage{stmaryrd}
\usepackage{url}

\newtheorem{feature}{Characteristic}

\newtheorem{theorem}{Theorem}
\newtheorem{proposition}{Proposition}
\newtheorem{conjecture}{Conjecture}
\newtheorem{definition}{Definition}
\newtheorem{example}{Example}
\newtheorem{fact}{Fact}

\newcommand{\qed} {\hfill{$\Box$}}
\newenvironment{proof}[1]{\noindent{\bf Proof#1: }}{\qed\medskip}

\newcommand{\eat}[1]{}

\newcommand{\ind}{\mathit{indegree}}

\newcommand{\arggraph}{\ensuremath{\mathds{A}}\xspace}
\newcommand{\arggraphprime}{\ensuremath{\mathds{A^\prime}}\xspace}
\newcommand{\argset}{\ensuremath{\mathcal{A}}\xspace}

\newcommand{\attack}{\ensuremath{\mathcal{R}_a}\xspace}
\newcommand{\support}{\ensuremath{\mathcal{R}_s}\xspace}

\newcommand{\bs}{\ensuremath{\texttt{S}}\xspace}
\newcommand{\ads}{\ensuremath{ADS_{\bs}}\xspace}

\newcommand{\accdegr}[3]%
	{\ensuremath{\texttt{Deg}
	^{\ifthenelse{\equal{#1}{}}{\bs}{#1}}%
	_{\ifthenelse{\equal{#2}{}}{\arggraph}{#2}}
	{\ifthenelse{\equal{#3}{}}{ }{(#3)}}
	\xspace}}
\newcommand{\accdegrvec}[2]%
	{\ensuremath{\texttt{Deg}
	^{\ifthenelse{\equal{#1}{}}{\bs}{#1}}%
	_{\ifthenelse{\equal{#2}{}}{\arggraph}{#2}}
	\xspace}}
\newcommand{\maxd}{\ensuremath{\texttt{Max}_\texttt{S}}\xspace}
\newcommand{\mind}{\ensuremath{\texttt{Min}_\texttt{S}}\xspace}
\newcommand{\neutrald}{\ensuremath{\texttt{Neutral}_\texttt{S}}\xspace}	

\newcommand{\defaultarggraph}{$\arggraph  =\langle {\argset, G, w}\rangle $\xspace}
\newcommand{\alternativearggraph}{	$\arggraph ^\prime =\langle {\argset ^\prime, G^\prime, w^\prime}\rangle $\xspace}
\newcommand{\attackers}[2]{\ensuremath{\texttt{Att}_{\ifthenelse{\equal{#1}{}}
				{\arggraph}{#1}}({#2})}}
\newcommand{\defenders}[2]{\ensuremath{\texttt{Def}_{\ifthenelse{\equal{#1}{}}
				{\arggraph}{#1}}({#2})}}
\newcommand{\supporters}[2]{\ensuremath{\texttt{Sup}_{\ifthenelse{\equal{#1}{}}
				{\arggraph}{#1}}({#2})}}
\newcommand{\wasa}{\textsc{wasa}\xspace}
\newcommand{\bwsa}{\textsc{bwsa}\xspace}
\newcommand{\backers}[2]{\ensuremath{\texttt{Back}_{\ifthenelse{\equal{#1}{}}
				{\arggraph}{#1}}({#2})}}
\newcommand{\detractors}[2]{\ensuremath{\texttt{Detr}_{\ifthenelse{\equal{#1}{}}
				{\arggraph}{#1}}({#2})}}
\newcommand{\parent}[2]{\ensuremath{\texttt{Parent}_{\ifthenelse{\equal{#1}{}}
				{\arggraph}{#1}}({#2})}}
\newcommand{\semdir}{\ensuremath{dir}}
\newcommand{\fsemdir}{\ensuremath{f^{\semdir}}\xspace}
\newcommand{\semsdir}{\ensuremath{sdir}\xspace}
\newcommand{\semrsig}{\ensuremath{rsig}\xspace}
\newcommand{\fsemrsig}{\ensuremath{f^{\semrsig}}\xspace}
\newcommand{\semrd}{\ensuremath{rd}}
\newcommand{\fsemrd}{\ensuremath{f^{\semrd}}}

\newcommand{\norm}[1]{\interleave #1\interleave}
\newcommand{\rownorm}[1]{\interleave #1\interleave_\infty}

\newcommand{\matr}[1]{\left(\begin{smallmatrix}#1\end{smallmatrix}\right)}
\renewcommand{\Pr}{\mathit{Pr}}

\begin{document}
\title{Bipolar Weighted Argumentation Graphs}

\author{Till Mossakowski \and Fabian Neuhaus \and
Otto-von-Guericke Universit\"at Magdeburg, Germany}

\date{}
\maketitle
\begin{abstract}
This paper discusses the semantics of weighted argumentation graphs that are bipolar, i.e.\ 
contain both attacks and support graphs. The work builds on previous work by 
Amgoud, Ben-Naim et. al.
 \cite{DBLP:conf/ijcai/AmgoudB16, amgoudevaluation}, which presents and compares several semantics 
for  argumentation graphs that
 contain only supports or only attacks relationships, respectively.  	
\end{abstract}

\section{Introduction}

In \cite{argSpaceExpl} we presented a prototype of a system that enables users to explore 
arguments for a given topic. This involves these  steps:
	\begin{enumerate}
		\item \emph{Argument identification.}
		 In the first step, arguments concerning a given topic are 	identified in a given text and
		 attacking and supporting relationships between
		  the propositions are established. The result is
		 an argumentation graph. In the future we hope to use argumentation mining techniques to 
		 automate this step. At this time, this is done manually by marking up some text. 
		 \item \emph{Initial plausibility assessment.} The propositions (represented as the nodes 
		 of the argumentation graph)
		  are assigned an initial plausibility  based on Web
		  searches. 
		 \item \emph{Acceptability degree calculation.} The acceptability degree of the propositions are calculated. This calculation depends on the plausibility assessment 
		 of the propositions and the attacks and supports between them. 
		 \item \emph{Recommendation.} Based on the acceptability degree calculation
		 the system determines an answer and the argument for it. 
		 \item \emph{User Interaction.} The user is able to explore the argument graph and 
		 may manually change the plausibility rating of any proposition. The system
		   recalculates the acceptability degrees and its recommendation. 
	\end{enumerate}

In this paper we focus on the third step: the calculation of acceptability degrees of the 
propositions of the arguments. For the sake of this paper we treat the propositions
as arguments in an argument network.\footnote{
This is a significant simplification, since it only allows us to represent convergent arguments 
in the sense of \cite{Walton2009}. Further, this approach does not consider that the strength of
a support or attack may be weighted as well; see section \ref{sec:future}.
}

Abstract argumentation has been extensively studied since Dung's pioneering
work \cite{DBLP:journals/ai/Dung95} on argumentation graphs featuring
an attack relation between arguments. While the Dung framework uses a
classical logic approach in the sense that arguments can be only true
or false, the framework has been generalised to gradual (or
rank-based, or weighted) argumentation graphs that assign real numbers
as weights to arguments.  These have been widely discussed, in
particular in
\cite{Cayrolc05,DBLP:conf/ictai/CayrolL11,DBLP:conf/sum/AmgoudB13a},
and also for the bipolar case (involving both attack and support
relations) \cite{DBLP:conf/ecsqaru/CayrolL05}. However, these works do
not considering initial weightings (also called initial plausibilities).  
Given that the result of the
initial plausibility assessment provide a continuous initial
weighting, the most relevant previous work is on the evaluation
argumentation graphs with support relationships by Leila Amgoud and
Jonathan Ben-Naim \cite{DBLP:conf/ijcai/AmgoudB16} and an unpublished
work by the same authors and two other authors on the evaluation of
argumentation graphs with attack relationships
\cite{amgoudevaluation}.

Since for our purposes we need to consider bipolar \cite{DBLP:journals/ijar/CayrolL13} argumentation graphs that contain both attack and support 
relationships, we need to generalise  the results of \cite{DBLP:conf/ijcai/AmgoudB16} and \cite{amgoudevaluation}. 
The result is a novel acceptability semantics for weighted argumentation graphs that contain both 
 attacks and supports between arguments. The outline of the paper follows the presentations in
 \cite{DBLP:conf/ijcai/AmgoudB16, amgoudevaluation}. 
  In section \ref{sec:basics} we introduce the basic notions. In section \ref{sec:characteristics} 
 we discuss the characteristics an acceptability semantics should have. These characteristics are characterised 
axiomatically. We discuss how these axioms relate to the axioms from 
 \cite{DBLP:conf/ijcai/AmgoudB16, amgoudevaluation}.  
  
In section \ref{sec:semantics} we discuss a semantics that meets the characteristics 
from section \ref{sec:characteristics}, show that it converges,
study some properties and derive a variant with
weights in the interval $(0,1)$.
In section \ref{sec:comparison} 
we compare our approach to  \cite{DBLP:conf/ijcai/AmgoudB16, amgoudevaluation} in more detail. Since a na\"{i}ve combination of the semantics \cite{DBLP:conf/ijcai/AmgoudB16, amgoudevaluation} to a bipolar one fails, we discuss two suitable modifications. Finally, in 
\ref{sec:future} we discuss some limitations of our approach and future work. 
 
The two main contributions of this work are as follows. Firstly, we
generalise the axiomatic framework of \cite{DBLP:conf/ijcai/AmgoudB16,
  amgoudevaluation} in various directions (bipolarity, unboundedness,
multi-graph characteristics), as well as strengthen it (partly much
stronger characteristics, as well as new ones like
Continuity). Secondly, we design an unbounded bipolar semantics for
weighted argumentation graphs that meets the developed
characteristics. This semantics meets the requirements that emerged
when developing a prototype of a system that enables users to explore
arguments for a given topic.

\section{Basic Concepts}\label{sec:basics}

In  \cite{DBLP:conf/ijcai/AmgoudB16, amgoudevaluation} argumentation graphs are represented as a 
set of weighted nodes, which represent the attacks, and a set of vertices, which represent an attack 
relationship or a support relationship, respectively. We choose an alternative representation.   An argument graph consists of three elements: a vector of arguments $\argset =a_1, \ldots a_n$, a matrix $G$ that determines the 
attack and support relationships between the arguments, and a weighting $w$ of the arguments which provides initial plausibilities.  
More specifically, $G$ is a square matrix of order $n$, which elements are either $-1, 0$ or $1$. 
Given a matrix $G$, if the element $g_{ij} = 1$, then this is intended to represent that the argument $a_j$ supports the
argument $a_i$; if $g_{ij} = -1$, then argument $a_j$ attacks $a_i$; and if $g_{ij} = 0$, then $a_j$ 
does neither support nor attack $a_i$. 
The vector $w$ assigns to each argument a real number to represent its initial plausibility. 
The larger
the $w(a)$ for some argument $a$ is, the larger its initial plausibility.  

Note that our approach  deviates
from the approach in \cite{DBLP:conf/ijcai/AmgoudB16, amgoudevaluation}, where
only values in the interval $[0,1]$ are considered. 
In this paper we decided to allow $\mathbb{R}$ as the value space, since 
there 
is no a priori reason why weights should always be restricted to 
the interval $[0,1]$, and, thus, we aim to support the 
more general case. 
For example, we plan to use hit counts as initial
plausibilities, which can be used directly without any normalisation. 
 $0$ is is the neutral value (neither plausible nor implausible).
Negative values denote implausibility (e.g.\ consider hit counts 
contradicting the argument). 
In section~\ref{sec:frenchSection} we show that the support argumentation graphs 
in \cite{DBLP:conf/ijcai/AmgoudB16} ranging in the interval $[0,1]$ may be considered as a special instance of \wasa{}. Within this approach $0$ is  the neutral value, and, hence, implausibility
cannot be expressed.
In section~\ref{sec:sigmoid-direct}  we show how our proposed semantics \label{br:sigmoiddirect}
may be adopted to support initial plausibilities and acceptability degrees in  
the interval $(0,1)$, which enables a direct comparison
to the semantics in  \cite{DBLP:conf/ijcai/AmgoudB16, amgoudevaluation}. 
One difference, though,  is that within this approach $\frac{1}{2}$ is the neutral value and that a value in the interval $(\frac{1}{2}, 0)$ expresses implausibility.

\begin{definition}[Weighted Attack/Support Argumentation Graph]
	A weighted  attack/support argumentation graph (\wasa) 
	is a triple \defaultarggraph{}, 	where 
	\begin{itemize}
		\item \argset is a vector of size $n$ (for some $n\in \mathbb{N}^+$), where all components of \argset are pairwise distinct. 
		\item $G=\{g_{ij}\}$ is a square matrix of order $n$ where $g_{ij}\in \{-1,0,1\}$, 
		\item $w$ is a vector in $\mathbb{R}^n$. 
	\end{itemize}	
\end{definition}
If \defaultarggraph{} is a \wasa and \arggraph is of size $n$, then $\arggraph$ consists of $n$ 
arguments. 

\begin{example}\label{ex:arggraph}
\

\begin{minipage}{0.5 \textwidth}
\[
\arggraph = \left \langle 
\left(\begin{smallmatrix}
a_1 \\
a_2 \\
a_3 \\
a_4
\end{smallmatrix}\right), 
\left(\begin{smallmatrix}
0 & 0 &  0 &  0 \\ 
1 & 1 & -1 & -1 \\ 
1 & 1 &  0 & -1 \\ 
0 & 0 &  0 &  0 
\end{smallmatrix}\right), 
\left(\begin{smallmatrix}
0.5 \\
2 \\
2 \\
1 
\end{smallmatrix}\right)
\right \rangle 
\]		
\end{minipage}
\begin{minipage}{0.3 \textwidth}
$$\xymatrix{
\underset{0.5}{a_1} \ar[r]\ar[d]   &  \underset{2}{a_2}  \ar@(r,u)\ar[dl]<0.5ex> \\
\underset{2}{a_3} \ar@{-{*}}[ur]<0.5ex>   &   \underset{1}{a_4}  \ar@{-{*}}[l] \ar@{-{*}}[u]
}$$
\end{minipage}	
\end{example}

The \wasa \arggraph in Example \ref{ex:arggraph} consists of four arguments, namely $a_1, a_2, a_3, a_4$. 
The second component of \arggraph determines that $a_1$ and $a_4$ are neither attacked nor supported. $a_1$ and $a_2$ support both $a_2$ and $a_3$. 
$a_2$ is attacked by $a_3$ and $a_4$, and $a_3$ is attacked by $a_4$. 
The third
component of $\arggraph$ assigns initial plausibilities to the arguments, namely the weights   $w(a_1)= 0.5, w(a_2) = 
2, w(a_3) = 2$ and $w(a_4) = 1$.
Example \ref{ex:arggraph} also contains a graphical representation of \arggraph, which represents 
support relationships as connections with an arrow head 
and attacks as connections with a round head. 

A \wasa is a representation of a set of arguments, their attack and support relationships, and the initial plausibility of the arguments. The question that this paper needs to address is: How do we calculate the acceptability of the arguments based on their initial plausibility and their relations? Following the terminology in
\cite{DBLP:conf/ijcai/AmgoudB16, amgoudevaluation}, an 
answer to this question is called an acceptability semantics:

\begin{definition}[Acceptability Semantics]
An acceptability semantics is a function \bs transforming any \wasa \defaultarggraph into a vector $\accdegr{}{}{}{}$ in $ \mathbb{R}^n$, where $n$ is the number of arguments in \arggraph.
For any argument $a_i$ in \argset, $\accdegr{}{}{a_i}$ is called the acceptability degree of $a_i$.
	\end{definition}

Obviously, there are many possible acceptability semantics. Example \ref{ex:nihlism} defines
the acceptability semantics $S^G$ that may have been embraced by the Greek sophist Gorgias, 
who believed that knowledge and communication is impossible. 
\begin{example}[Gorgias Semantics]
	\label{ex:nihlism}
	$S^{G}$ is the function such that for, any \wasa \arggraph that consists of $n$ arguments, 
	$S^{G}(\arggraph)$ is a vector of size $n$ such that $S^{G}(\arggraph) =
\left(\begin{smallmatrix}
0 \\
\vdots \vspace{5pt}\\
0 \\
\end{smallmatrix}\right )$.
\end{example}


According the Georgias Semantics any argument is equally acceptable and unacceptable, thus,  $\accdegr{S^G}{}{a} = 0$ for any argument $a$ in any \wasa  $\arggraph$.

Most people would probably agree that $S^{G}$ does not provide us with a useful tool for analysing 
argumentations. However, it raises the questions what requirements a suitable 
acceptability semantics should meet. We will discuss this question in the section \ref{sec:characteristics}. 


	
\section{Notation and auxiliary definitions}
	Unless otherwise specified \arggraph is a \wasa such that \defaultarggraph. If $a_1, \ldots, a_n$  are the components of  $\argset$  we denote by 
	\begin{itemize}
		\item 	\attackers{}{a_i} the set of all attackers of $a_i$ in 
		\arggraph, that is $\attackers{}{a_i} = \{ a_j  | \ g_{ij} = -1\} $; 
		\item 	\supporters{}{a_i} the set of all supporters of $a_i$ in \arggraph, 
		that is $\supporters{}{a_i} = \{ a_j  | \ g_{ij} = 1\} $;

		\item $\backers{}{a_i}$ and  $\detractors{}{a_i}$  are the sets of the backers of $a_i$ and the set of the detractors of $a_i$. They are defined recursively as 	the set of all arguments that directly or indirectly support $a_i$ (e.g., attacking an attacker of $a_i$) and, respectively,  	the set of all arguments that directly or indirectly attack $a_i$ (e.g., supporting an attacker of $a_i$). 
		Thus, $\backers{}{a_i}$ and $\detractors{}{a_i}$ are the minimal sets such that the following equations hold:
		\begin{multline}
			\backers{}{a_i} = \supporters{}{a_i} \cup \{ a_j | \ \exists 
			 x: a_j \in \supporters{}{x} \wedge x \in  \backers{}{a_i}\} \notag\\
			 %
			 \cup   \{ a_j | \ \exists 
			 			 x: a_j \in \attackers{}{x} \wedge x \in  \detractors{}{a_i}\} \notag
		\end{multline}
		\begin{multline}
			\detractors{}{a_i} = \attackers{}{a_i} \cup \{ a_j | \ \exists 
			 x: a_j \in \attackers{}{x} \wedge x \in  \backers{}{a_i}\} \notag\\
			 %
			 \cup   \{ a_j | \ \exists 
			 			 x: a_j \in \supporters{}{x} \wedge x \in  \detractors{}{a_i}\} \notag
		\end{multline}
		%

              \item $\parent{}{a_i}$ is the $i$th matrix row
                $(g_{i1},\ldots,g_{in})$ of $G$. It contains the
                parents of $a_i$ in the argument graph and hence
                combines the information of $\supporters{}{a_i}$ and
                $\attackers{}{a_i}$ in one vector.
\end{itemize}

\begin{definition}[Influence]
Given a \wasa \defaultarggraph and a vector $v\in\mathbb{R}^n$ (e.g.\
$v$ could be $w$), the \emph{influence} of $v$ on $a_i$ is defined as
the number
$$ \sum_{b\in \supporters{}{a_i}}v(b)\ - 
					\sum_{c\in \attackers{}{a_i}} v(c)
= \parent{}{a_i}v$$
The influence of $v$ in general is computed as the vector of the
individual influences:
$$ \left(  
 \begin{smallmatrix}
   \parent{}{a_1}v\\
   \vdots\\
   \parent{}{a_n}v
 \end{smallmatrix}\right)= Gv$$
\end{definition}
Note that supporters and attackers cancel each other out when
computing the influence. Moreover, support by an implausible argument
(weighted negatively) behaves like an attack, an vice verse, an attack
by an implausible argument behaves like a support. This is called
reverse impact, see Characteristics~\ref{f:rev-impact} below.

\begin{definition}[Isomorphism]
  Let   \defaultarggraph and \alternativearggraph be two \wasa, such that:
\[ 
	\argset = 
	\left(\begin{smallmatrix}
	a_1 \\
	\vdots\\
	a_n
	\end{smallmatrix}\right),
	\argset ^\prime = \left(\begin{smallmatrix}
  	a_1^\prime \\
  	\vdots\\
  	a_n^\prime
  	\end{smallmatrix}\right), 
    G =
  	\left(  \begin{smallmatrix}
	     g_{1 1} &  \cdots & g_{1 n} \\
	     \vdots  & \ddots & \vdots  \\
	    g_{n 1}  & \cdots & g_{n n}
	    \end{smallmatrix}\right) , 
    G^\prime =
     \left(  \begin{smallmatrix}
       g_{1 1}^\prime &  \cdots & g_{1 n}^\prime \\
       \vdots  & \ddots & \vdots  \\
      g_{n 1}^\prime  & \cdots & g_{n n}^\prime 
      \end{smallmatrix}\right)
\]

%

  %

%
  An isomorphism from  \arggraph
	to $\arggraph^\prime$ is a bijective function $f$ from \argset to $\argset^\prime$ such
	that the following holds, for any $a_i, a_j$,
	\begin{itemize}
		\item 	$ w( a_{i} ) = w^\prime(f( a_{i} ))$,
		\item if $f(a_i) = a^\prime _k$ and $f(a_j) = a^\prime _m$, 
				then $g_{i j} = g_{k m}^\prime  $.
	\end{itemize}
\end{definition}

\begin{definition}[Union]
  Let   \defaultarggraph and \alternativearggraph be two \wasa 
   such that $\argset$ and $\argset^\prime$ do not share a component and   
\[ 
	\argset = 
	\left(\begin{smallmatrix}
	a_1 \\
	\vdots\\
	a_n
	\end{smallmatrix}\right),
	\argset ^\prime = \left(\begin{smallmatrix}
  	a_1^\prime \\
  	\vdots\\
  	a_m^\prime
  	\end{smallmatrix}\right), 
    G =
  	\left(  \begin{smallmatrix}
	     g_{1 1} &  \cdots & g_{1 n} \\
	     \vdots  & \ddots & \vdots  \\
	    g_{n 1}  & \cdots & g_{n n}
	    \end{smallmatrix}\right) , 
    G^\prime =
     \left(  \begin{smallmatrix}
       g_{1 1}^\prime &  \cdots & g_{1 m}^\prime \\
       \vdots  & \ddots & \vdots  \\
      g_{m 1}^\prime  & \cdots & g_{m m}^\prime 
      \end{smallmatrix}\right),
\]

\[ 
	w = 
	\left(\begin{smallmatrix}
	w_1 \\
	\vdots\\
	w_n
	\end{smallmatrix}\right),
	w ^\prime = \left(\begin{smallmatrix}
  	w_1^\prime \\
  	\vdots\\
  	w_m^\prime
  	\end{smallmatrix}\right)
\]

   The union $\arggraph	\oplus \arggraph ^\prime = \langle \argset ^{\dagger}, G ^{\dagger}, w^{\dagger}\rangle $  of \arggraph and \arggraphprime is defined as follows:
   \[
	\argset^{\dagger} = 
	\left(\begin{smallmatrix}
	a_1 \\
	\vdots\\
	a_n\\
  	a_1^\prime \\
  	\vdots\\
  	a_m^\prime
	\end{smallmatrix}\right),
    G 	^{\dagger}=
  	\left(  \begin{smallmatrix}
	     g_{1 1} &  \cdots & g_{1 n}& 0 &  \cdots  & 0\\
	     \vdots  & \ddots & \vdots  &  \vdots &  \ddots  &  \vdots\\
	    g_{n 1}  & \cdots & g_{n n} & 0 &  \cdots  & 0\\ 
      0&\cdots & 0 & g_{1 1}^\prime &  \cdots & g_{1 m}^\prime \\
      \vdots & \ddots &\vdots &  \vdots  & \ddots & \vdots  \\
      0 &\cdots & 0 & g_{m 1}^\prime  & \cdots & g_{m m}^\prime 	  
	    \end{smallmatrix}\right) , 	
 	w^{\dagger} = 
 	\left(\begin{smallmatrix}
 	w_1 \\
 	\vdots\\
 	w_n\\
   	w_1^\prime \\
   	\vdots\\
   	w_m^\prime
 	\end{smallmatrix}\right)
   \] 
To improve readability we will in the rest of the paper use a more compact notation, which does not 
list the individual components but refers to the matrixes that are merged. ("0" represent a zero matrix of appropriate dimensions, that is $n \times m$ and $m\times n$, respectively.)
   \[
	\argset^{\dagger} = 
	\left(\begin{smallmatrix}
	\argset \\
	\argset^\prime
	\end{smallmatrix}\right),
    G 	^{\dagger}=
  	\left(  \begin{smallmatrix}
		G & 0 \\
		0 & G^\prime
	    \end{smallmatrix}\right) , 	
 	w^{\dagger} = 
 	\left(\begin{smallmatrix}
 	w \\
 	w^\prime\\
 	\end{smallmatrix}\right)
   \] 
   	
\end{definition}
\section{Characteristics of Acceptability  Semantics
}	\label{sec:characteristics}
There are many possible acceptability semantics that one 
may consider for bipolar argumentation graphs. 
As Example \ref{ex:nihlism} illustrates, some of them are not useful. 
Thus, the question arises, which characteristics an 
acceptability semantics should have to be any good? 

In 
\cite{DBLP:conf/ijcai/AmgoudB16, amgoudevaluation} the authors enumerate several 
desirable characteristics, which they state axiomatically. These characteristics are distinguished between
mandatory and optional. In \cite{amgoudevaluation} eleven mandatory characteristics are discussed, 
\cite{DBLP:conf/ijcai/AmgoudB16} contains eleven similar%
\footnote{Technically, the characteristics are all different, since 
argument graphs  in \cite{DBLP:conf/ijcai/AmgoudB16} 
involve (only) support relationships and the argument graphs  in \cite{amgoudevaluation}
involve (only) attack relationships. Further, there are additional technical differences in their 
axiomatisations. However, for the purposes of this paper we disregard these differences.}
characteristics and two additional ones (Monotony, Boundedness).
%

%
%
%

\begin{table}
\centering
\begin{small}
\begin{tabular}{|l|l|l|}\hline
	    Support Graphs in  \cite{DBLP:conf/ijcai/AmgoudB16}  & Attack Graphs in \cite{amgoudevaluation} & Our notion\\ \hline
\multicolumn{3}{|c|}{Mandatory characteristics}\\ \hline
     Anonymity &  Anonymity  &  Anonymity\\ 
	 Independence &  Independence &  Independence\\ 
	 Equivalence &  Equivalence  &  Equivalence\\ 
	 Non-dilution &  Directionality &  Directionality\\ 
	 Minimality &  Maximality  & Conservativity\\ 
	Coherence & Proportionality & Initial Monotony \\ 
	 Dummy &  Neutrality &  Neutrality \\ 
	 	Monotony & -- & Parent Monotony\\ 
	Strengthening &  Weakening & Impact\\ 
	 Counting &  Counting & Impact \\
	 	Reinforcement &  Reinforcement &  Reinforcement \\ 
	Strengthening Soundness & Weakening Soundness & Causality\\ 
	Boundedness & -- & Stickiness \\ 
	-- & -- & Neutralisation\\ %
		-- & -- & Continuity \\  %
		-- & -- &  Interchangeability \\ \hline %
\multicolumn{3}{|c|}{Optional characteristics}\\ \hline
		-- & -- &  Linearity \\ 
		-- & -- &  Reverse impact \\ 
	-- & -- & Boundedness \\ \hline

\end{tabular}
\end{small}

\caption{Characteristics of acceptability semantics in  \cite{DBLP:conf/ijcai/AmgoudB16, amgoudevaluation}
 }
	\label{tab:characteristics}
\end{table}

In this section we discuss these characteristics and define them within our framework. 
Table \ref{tab:characteristics} provides an overview over the mandatory characteristics in 
\cite{DBLP:conf/ijcai/AmgoudB16, amgoudevaluation}  and maps them to the terminology that we 
use in this paper.

The definition of the mandatory characteristics in 
\cite{DBLP:conf/ijcai/AmgoudB16, amgoudevaluation}  within our framework
involves different kind of changes. First, the definitions in \cite{DBLP:conf/ijcai/AmgoudB16, amgoudevaluation} assume that weightings and acceptability degrees are within the interval $[0, 1]$, whereas  
we allow arbitrary real numbers. 
Second, we need to account for the fact that a \wasa may contain both attack and support  relationships. 
Third, because of the way the characteristics were formulated in  \cite{DBLP:conf/ijcai/AmgoudB16, amgoudevaluation} some of them allowed for unintended semantics. These three points will be detailed below. We formulated the axioms in a way that captures the intended characteristic in a more general way. 

The definition of the characteristics depends on three parameters. The
first one is a \emph{neutral acceptability degree}
$\neutrald$. Attacks or supports by arguments with the neutral
acceptability degree will have no effect. If not stated otherwise, we
assume $\neutrald=0$. The other two parameters are the minimum and
maximum acceptability degrees. They can be derived from the
acceptability degree space as follows: Let \bs be an acceptability
semantics. Its \emph{acceptability degree space} $\ads = \{x \ | \ x =
\accdegr{}{}{a} \mbox{ \ for some \wasa \defaultarggraph and $a$ in
  $\argset$}\}$.  If there is some $x\in \ads $ such that $x \geq y$
for all $y\in \ads $, then $\ads $ is bounded from above and its
\emph{maximum acceptability degree} $\maxd = x$. Otherwise, $\maxd$ is
undefined.  If there is some $x\in \ads$ such that $x \leq y$ for all
$y\in \ads$, then $\ads$ bounded from below and its \emph{minimum
  acceptability degree} $\mind = x$. Otherwise, $\mind$ is undefined.

The neutral value also plays a role in ``neutralising'' arguments as follows:
\begin{definition}[Isolation]
	Let \defaultarggraph be a \wasa such that $a_1, \ldots, a_n$ are the components of
	\argset{}.  The \emph{isolation}  $\arggraph |_{a_i}$  of $a_i$ in \arggraph  $(1\leq i\leq n)$
	is defined as follows 
 \[
 \arggraph |_{a_i} = 
 \left \langle \argset,
 \left(  \begin{smallmatrix}
 			       g_{1 1}^\prime &  \cdots & g_{1 n}^\prime \\
 			       \vdots  & \ddots & \vdots  \\
 			      g_{n 1}^\prime  & \cdots & g_{n n}^\prime
 			      \end{smallmatrix}\right) 
				  , 
				   \left( \begin{smallmatrix}
				  	w_{1}^\prime \\
				   	\vdots \\
				   	w_{n}^\prime  
				  \end{smallmatrix}\right)
				 \right \rangle \ \mathit{where}
 		\begin{cases}
 			g_{j k}^\prime = \neutrald & \mathit{if} \ j = i \mathit{\ or \ } k = i\\
 			g_{j k}^\prime = g_{j k} & \mathit{otherwise} \\
		 	w_j^\prime  =\neutrald & \mathit{if} \ a_j = a_i \\
			w_j^\prime  = w_j & \mathit{otherwise}
 		\end{cases}
 \]
If $c_1, \ldots , c_n $ are arguments in \argset, then  $\arggraph |_{c_1, \ldots , c_n}$ is defined as 
$(\ldots((\arggraph |_{c_1})|_{c_2})\ldots )|_{c_n}$.

		%
		%
		%

\end{definition}

\subsection{Mandatory Characteristics}\label{sec:mandatory}


\emph{Anonymity} implies that the identity of an argument (or its internal structure) has no impact on 
an acceptability degree semantics. \emph{Independence} requires that the acceptability degree of 
an argument is influenced only by  arguments that are (directly or indirectly) connected to it.  
These definitions are, modulo trivial changes, identical with the 
corresponding definitions in  \cite{DBLP:conf/ijcai/AmgoudB16, amgoudevaluation}.
\begin{feature}[Anonymity]
	A semantics \bs satisfies \emph{Anonymity} iff, for any two \wasa 
	\defaultarggraph and \alternativearggraph and 	
	for any isomorphism $f$ from \arggraph to $\arggraph^\prime$, the following property holds: 
	for any $ a$  in \argset,  
	$ \accdegr{}{}{a}= \accdegr{}{\arggraph^\prime}{f(a)}$.
\end{feature}

\begin{feature}[Independence]
A semantics \bs satisfies \emph{Independence} iff, for any two \wasa 
	\defaultarggraph and \alternativearggraph such that 
    $\argset $ and $ \argset ^\prime$ do not share a component, the following property 
	holds: for any $a$ in $\argset$, $\accdegr{}{}{a} = \accdegr{}{\arggraph \oplus \arggraph 
	 ^\prime}{a}$. 
	
\end{feature}

\emph{Equivalence} requires that if two arguments start out with the 
same  initial plausibility and if they share the same degree of attack and support, they have
the same acceptability degree. This is  in the same spirit as Anonymity, the major difference is 
that Anonymity compares arguments across different \wasa{}, while  Equivalence is about the arguments within one \wasa{}.
The definition is a straight forward combination of the corresponding definitions in  \cite{DBLP:conf/ijcai/AmgoudB16, amgoudevaluation}.
\begin{feature}[Equivalence]
A semantics \bs satisfies \emph{Equivalence} iff, for any weigh\-ted argumentation graph 
\defaultarggraph{} and for any $a, b$ in $\argset $, if 
	\begin{itemize}
		\item $w(a) = w(b)$, 
		\item there exists a bijective function $f$ from  \attackers{}{a} to \attackers{}{b} 
		such that $\forall x \in \attackers{}{a}$, $\accdegr{}{}{x}= \accdegr{}{}{f(x)}$, 
		\item there exists a bijective function $g$ from  \supporters{}{a} to \supporters{}{b} 
		such that $\forall x \in \supporters{}{a}$, $\accdegr{}{}{x}= \accdegr{}{}{g(x)}$, 
	\end{itemize}
		then $\accdegr{}{}{a} = \accdegr{}{}{b}  $. 	
\end{feature}


\emph{Directionality} captures the idea that attack and support are directed relationships, that is 
the attacker (supporter) influences the acceptability degree of the attacked (supported), but 
not vice versa. Thus, assume \arggraph is a \wasa and one adds a new attack  
(or support) relationship from $a_i$ to $a_j$, then this should only affect the acceptability of 
$a_j$ and arguments that $a_j$ directly or indirectly attacks or supports. To put it in a different 
way, all arguments (other than $a_j$) that do not have $a_j$ as backer or detractor, should 
not be affected by adding the new attack (support, respectively) relationship and their acceptability degree should not change.   

\begin{feature}[Directionality]
	A semantics \bs satisfies Directionality iff, for any two \wasa 
	\defaultarggraph and 
    $\arggraph^\prime =\langle {\argset , G ^\prime, w  
}\rangle $ the following holds: if  
	 $G$ and $G^\prime$ are of order $n$ (for some $n\in \mathbb{N}^+ $)  and 
	 there exists $i,j\in \mathbb{N}^+ $  such that 
	\begin{itemize}
		\item $g_{ji} = 0$,
		\item $g^\prime_{ji} \neq 0$,
		\item for any $k, m$: if $k\neq i$ or $m\neq j$, then $g_{mk} = g^\prime_{mk}$,
	\end{itemize}
then for all $x$ in $\argset$: if $x \neq a_j$ and $a_j\not \in \backers{}{x}$ and 
$a_j\not \in \detractors{}{x}$, then $\accdegr{}{}{x} = \accdegr{}{\arggraph^\prime}{x}$.

%
%
%

\end{feature}
Our definition of Directionality translates the  corresponding definition in \cite{amgoudevaluation} into our our matrix-based approach. 
Both are more general than Non-dilution in \cite{DBLP:conf/ijcai/AmgoudB16}.

Conservativity expresses that, given any lack of supports or attacks, the acceptability degree 
of an argument should be identical to its initial plausibility. It combines  Minimality in \cite{DBLP:conf/ijcai/AmgoudB16}  with  Maximality in \cite{amgoudevaluation}. 

\begin{feature}[Conservativity]
A semantics \bs satisfies \emph{Conservativity}
 iff for any \wasa \defaultarggraph, for any argument $a$ in $ \argset$, if 
$\attackers{}{a} = \supporters{}{a} = \emptyset$, then $\accdegr{}{}{a} = w(a)$.
\end{feature}

In \cite{DBLP:conf/ijcai/AmgoudB16}  the Coherence axiom is explained as follows: 
``[\ldots ] the impact  of support is proportional to the basic strength of its target''. The same 
characteristic is 
named  `Proportionality'  in  \cite{amgoudevaluation}. However, neither
of the axiomatisations  
in \cite{DBLP:conf/ijcai/AmgoudB16} and
 \cite{amgoudevaluation} represents proportionality in its usual sense. They 
rather require that an increase in the weights leads to an increase of the acceptability degree.
Since it is about monotony in the initial plausibility, we call it
\emph{Initial Monotony} (in contrast to Parent Monotony introduced later on).

\begin{feature}[Initial Monotony]
  A semantics \bs satisfies \emph{Initial Monotony} iff, for any
  \wasa \defaultarggraph and for any arguments $a,b$ in $\argset$, if
\begin{itemize}
\item $\parent{}{a}=\parent{}{b}$, and
\item $w(a)>w(b)$,
\end{itemize}
then
$$\accdegr{}{}{a}>\accdegr{}{}{b} \text{ or }\accdegr{}{}{a}=\accdegr{}{}{b}=\maxd \text{ or }\accdegr{}{}{a}=\accdegr{}{}{b}=\mind.$$
(Note that the equation $\accdegr{}{}{a}=\accdegr{}{}{b}=\maxd$ is taken to
be false if $\maxd$ does not exist, and similarly for $\mind$.)
\end{feature}

%
%
%

The initial plausibilities and the acceptability degrees of arguments are expressed as real numbers. Numbers (much) greater than $\neutrald$ represent high  plausibility and a high acceptability, 
respectively, of an argument.  Numbers (much) less than $\neutrald$ represent high implausibility and a strong inadequateness, respectively. 
$\neutrald$ plays a special role as the middle ground. An initial plausibility of $\neutrald$ means that the 
argument is neither plausible nor implausible, and an acceptability degree of $\neutrald$ means that 
within the given \wasa there is neither grounds for accepting nor for rejecting the argument.

\emph{Neutrality} expresses that, given an argument $a$ with an  acceptability degree of $\neutrald$, 
one can remove all attack and support relationships that $a$ is involved in, since $a$ has no 
impact on the acceptability degrees of rest of the arguments. Together with Independence this 
implies that arguments with an acceptability degree of $\neutrald$ can be eliminated from a \wasa without 
changing the acceptability degrees of the other arguments. 

\begin{feature}[Neutrality] \label{feat:neutrality}
	A semantics \bs satisfies Neutrality iff, for any \wasa
	\defaultarggraph the following holds: if there is an argument $a$ in \argset such that 
	$\accdegr{}{}{a}= \neutrald$, 
	then $\accdegr{}{}{} = \accdegr{}{\arggraph |_{a}}{}  $.
 \end{feature}

Our definition of Neutrality is wider applicable than the corresponding notions in \cite{amgoudevaluation} and \cite{DBLP:conf/ijcai/AmgoudB16}.
 Example \ref{ex:neutrality} (with $\neutrald=0$) illustrates 
one difference to \cite{DBLP:conf/ijcai/AmgoudB16, amgoudevaluation}: for any semantics $S$ that exemplifies 
Conservativity and Neutrality,  $\accdegr{}{}{ a_2}  = 1$. 
(Because 
of Conservativity $\accdegr{}{}{ a_1}= 0  $, hence 
neutrality implies that $\accdegr{}{}{ a_2} = \accdegr{}{\arggraphprime}{ a_2} $, and thus, by 
Conservativity,  $ \accdegr{}{\arggraphprime}{ a_2} = 1$.)
 In  \cite{DBLP:conf/ijcai/AmgoudB16, amgoudevaluation} the Neutrality and Minimality (or
Maximality, respectively) would not entail that $\accdegr{}{}{ a_2}  = 1$, because 
Neutrality is defined in a way that only compares acceptability degrees within one argumentation graph.

\begin{example}\label{ex:neutrality}
\ \newline
$\arggraph \: = \xymatrix{
\underset{0}{a_1} \ar[r]\   &  \underset{1}{a_2}   
}$
\quad

\noindent $\arggraphprime = \xymatrix{
\underset{0}{a_1} \   &  \underset{1}{a_2}   
}$

\end{example}


\emph{Parent Monotony} requires that, for any given argument $a$ in a \wasa, if one weakens or removes attackers of $a$ or strengthens or adds supporters of $a$, then this leads to a stronger or equal acceptability degree of $a$.  
Our formalisation of Parent Monotony significantly generalises the notion of Monotony in \cite{DBLP:conf/ijcai/AmgoudB16}.

\begin{feature}[Parent Monotony]
	A semantics \bs satisfies \emph{Parent Monotony} iff, for any two \wasa 
		\defaultarggraph and \alternativearggraph
     and  any argument $a$ which is both in $\argset $ and in $ \argset ^\prime$,  if
	\begin{enumerate}
			\item $w(a) = w^\prime (a)$ \label{monotony:w}
			\item 		$ \attackers{\arggraph ^\prime}{a} \subseteq \attackers{}{a}$ and
			 		$\supporters{}{a} \subseteq \supporters{\arggraph ^\prime }{a} $, \label{monotony:subset}
			\item 	for all $x \in \attackers{\arggraph ^\prime}{a}$, 
					$\accdegr{}{\arggraph ^\prime }{x} \leq \accdegr{}{}{x} $, \label{monotony:att}
    		\item for all $x \in \supporters{}{a}$, 		
 				$\accdegr{}{}{x} \leq \accdegr{}{\arggraph ^\prime}{x} $,  \label{monotony:deg}				
	\end{enumerate}
		then $\accdegr{}{}{a}  \leq \accdegr{}{\arggraph ^\prime}{a} $.%
\end{feature}

\emph{Impact} requires that adding a new supporting argument (with positive acceptability degree) strengthens (and thus has an impact on) the supported argument. Further, the opposite is true for adding a 
new attacking argument. Impact generalises Weakening and Counting in \cite{amgoudevaluation} and Strengthening and Counting
in \cite{DBLP:conf/ijcai/AmgoudB16}.\footnote{In \cite{DBLP:conf/ijcai/AmgoudB16} 
Strengthening  requires that the acceptability degree of a supported argument is higher than its 
initial plausibility. Counting requires that any additional support increases the acceptability 
more. Impact covers both given that Conservativity entails that in the absence of attackers and 
supporters the acceptability degree of an argument is identical to its initial plausibility. 
The analog is true for Strengthening and Counting in \cite{amgoudevaluation}.}
\begin{feature}[Impact]
	A semantics \bs satisfies \emph{Impact}
	 iff, for any \wasa 
		\defaultarggraph  
	 and 
	any arguments $a, b$ in $\argset$ such that  $\accdegr{}{}{b} > \neutrald$:	
	\begin{itemize}
		\item  If
		\begin{enumerate}
			\item  $ b \in \attackers{}{a}$ and
			\item $b\not\in \backers{}{a}$ 
		\end{enumerate}  
	then $\accdegr{}{}{a} < \accdegr{}{\arggraph|_{b} }{a} $  or  $\accdegr{}{}{a}=\accdegr{}{\arggraph|_{b}}{b}=\mind$,
		\item  If  
		\begin{enumerate}
			\item $ b \in \supporters{}{a}$ and
			\item $b\not\in \detractors{}{a}$,
		\end{enumerate}
	then $ \accdegr{}{\arggraph|_{b} }{a} < \accdegr{}{}{a} $ or $\accdegr{}{}{a}=\accdegr{}{\arggraph|_{b}}{a}=\maxd$.
	\end{itemize}

\end{feature}

\emph{Reinforcement} requires that if an attacker of an argument is weakened or 
a supporter is strengthened, then, ceteris paribus, the acceptability degree of the 
argument increases. Dually, if an attacker of an argument is strengthened  
or a supporter is weakened, then, ceteris paribus, the acceptability degree of 
the argument decreases. Reinforcement generalises the corresponding axioms in \cite{DBLP:conf/ijcai/AmgoudB16, amgoudevaluation} by considering both attacks and supports 
and by comparing acceptability degrees of arguments in different \wasa{}. Impact and 
Reinforcement together correspond to a kind of `Strict Parent Monotony'.

\begin{feature}[Reinforcement]
	A semantics \bs satisfies \emph{Reinforcement} iff, for any two \wasa 
	\defaultarggraph and \alternativearggraph and argument $a$ 
	 such that: 
		 $a$ 	is both in $\argset$ and $\argset^\prime$, and
		 $w(a) = w^\prime(a)$,
		 		$\attackers{}{a} = \attackers{\arggraph ^\prime}{a} $, and
			 		$\supporters{}{a} = \supporters{\arggraph ^\prime}{a} $		
	the following holds:\\ 				
	\begin{enumerate}
		\item If 
		\begin{itemize}
			\item for all $x \in \attackers{}{a}$, 
					$\accdegr{}{}{x} \leq \accdegr{}{\arggraph ^\prime}{x} $,
			\item for all $x \in \supporters{}{a}$, 		
					$\accdegr{}{}{x} \geq \accdegr{}{\arggraph ^\prime}{x} $,
					
			\item there is some $b$ such that either 
				\begin{itemize}
					\item 			$b \in \attackers{}{a}$ and 
					$\accdegr{}{}{b} < \accdegr{}{\arggraph ^\prime}{b} $ or 
					\item 	
					$b \in \supporters{}{a}$ and 
								$\accdegr{}{}{b} > \accdegr{}{\arggraph ^\prime}{b} $,
				\end{itemize}
	\end{itemize}	
	then $\accdegr{}{}{a} > \accdegr{}{\arggraph ^\prime}{a}$ or $\accdegr{}{}{a}=\accdegr{}{\arggraph ^\prime}{a}=\maxd$.
	
	\item If 
	\begin{itemize}
		\item for all $x \in \attackers{}{a}$, 
				$\accdegr{}{}{x} \geq \accdegr{}{\arggraph ^\prime}{x} $,
		\item for all $x \in \supporters{}{a}$, 		
				$\accdegr{}{}{x} \leq \accdegr{}{\arggraph ^\prime}{x} $,
				
		\item there is some $b$ such that either 
			\begin{itemize}
				\item 			$b \in \attackers{}{a}$ and 
				$\accdegr{}{}{b} > \accdegr{}{\arggraph ^\prime}{b} $ or 
				\item 	
				$b \in \supporters{}{a}$ and 
							$\accdegr{}{}{b} < \accdegr{}{\arggraph ^\prime}{b} $,
			\end{itemize}
	\end{itemize}			
		then $\accdegr{}{}{a} < \accdegr{}{\arggraph ^\prime}{a}$ or $\accdegr{}{}{a}=\accdegr{}{\arggraph ^\prime}{a}=\mind$.
	\end{enumerate}
			
\end{feature}

Strengthening Soundness in \cite{DBLP:conf/ijcai/AmgoudB16} and Weakening Soundness in 
\cite{amgoudevaluation} express that any difference between an initial plausibility and 
the acceptability degree of an argument is  caused by some supporting (attacking, respectively) argument. We call this characteristic \emph{Causality}. 

\begin{feature}[Causality]
	A semantics \bs satisfies \emph{Causality} iff, for any \wasa 
		\defaultarggraph 
     and  any argument $a$ in $\argset $,  if
	 $\accdegr{}{}{a} \neq w(a)$, then there exists an argument $b$ in $\argset$ such that
	 $\accdegr{}{}{b} \neq \neutrald$ and $b \in \attackers{}{a} \cup \supporters{}{a}$.  
\end{feature}
\begin{theorem}\label{thm:Causality}
Conservativity and Neutrality together imply Causality. 	
\end{theorem}

The idea behind \emph{Boundedness} in \cite{DBLP:conf/ijcai/AmgoudB16} is the following: 
if an  argument $a$ has  an acceptability degree of 1 and $b$'s support is stronger than $a$'s 
support, then the acceptability degree of $b$ is also 1. 
The name of this characteristic is somewhat of a misnomer, since its definition does not entail that 
the argument degree space is bounded in the usual mathematical sense. E.g., consider a semantics 
\bs such that \ads is the open interval $(0,1)$. Since \bs assigns the acceptability degree of 1 to no argument, 
the condition for Boundedness in \cite{DBLP:conf/ijcai/AmgoudB16} is met trivially, but  the open interval $(0,1)$ is not `bounded' in the usual mathematical sense of the word.

To avoid any possible confusion, we are going to define Boundedness in its traditional sense \ 
and rename the characteristic from \cite{DBLP:conf/ijcai/AmgoudB16} into \emph{Stickiness}.
Note that Stickiness presupposes that the maximum acceptability degree is 1. Thus, Stickiness
is not true for arbitrary \wasa. In section \ref{sec:frenchSection} we introduce  Stickiness for a subset of \wasa and show that it is entailed by the mandatory characteristics in this section;  
see  page \pageref{feat:stickiness}. Since Boundedness is not a 
mandatory characteristic we discuss it in section \ref{sec:optional}.

Up to this point all characteristics in this section are (more or less loosely) based on the 
mandatory characteristics that are discussed in 
\cite{DBLP:conf/ijcai/AmgoudB16, amgoudevaluation}. We add several new mandatory characteristics, namely 
 Neutralisation, Continuity, and Interchangeability.

\emph{Neutralisation} is concerned with the relationship between attacks and supports. If 
 argument $a_m$ is attacked by $a_l$ and supported by $a_k$ and the acceptability degrees of $a_l$ and $a_k$ are identical, then $a_k$ and $a_l$ neutralise each other (with respect to $a_m$). 
 Hence, if one removes both the attack from $a_l$ on   $a_m$ and the support of $a_k$ for $a_m$, then, ceteris paribus, the acceptability degrees of the arguments in the \wasa remain unchanged.

\begin{feature}[Neutralisation] 
  	A semantics \bs satisfies \emph{Neutralisation} iff, for any \wasa
  	\defaultarggraph and any components $a_k, a_l, a_m$ in $\argset$, if 
	\begin{itemize}
		\item  $a_k \in \attackers{}{a_m}$, 
		\item $a_l \in \supporters{}{a_m}$,  		
		\item  $\accdegr{}{}{a_k} = \accdegr{}{}{a_l}   $,
		\item $\arggraph^ \prime  =\langle {\argset , G^\prime, w  }\rangle $ and
			\[ 
			    G^\prime =
			     \left(  \begin{smallmatrix}
			       g_{1 1}^\prime &  \cdots & g_{1 m}^\prime \\
			       \vdots  & \ddots & \vdots  \\
			      g_{m 1}^\prime  & \cdots & g_{m m}^\prime 
			      \end{smallmatrix}\right)
				  \mathit{where}
				\begin{cases}
 				   	g^\prime_{ij} = 0 & \mathit{if \ } j= m \mathit{\ and \ } i = k \\
 				   	g^\prime_{ij} = 0 & \mathit{if \ } j= m \mathit{\ and \ } i = l \\
 				   	g^\prime_{ij} = g_{ij} & \mathit{otherwise}\\
				\end{cases}
			\]

	\end{itemize}
	then  $\accdegr{}{}{} = \accdegr{}{\arggraph ^\prime}{}$.
\end{feature}

Given a semantics $S$ that meets Neutralisation and Conservativity, the argument $a_2$ in Example \ref{ex:neutralisation}  has an acceptability degree  $\accdegr{}{}{a_2} = w(a_2) = 3$. Because the attack of $a_1$ and the support of $a_3$ neutralise each other.

\begin{example}\label{ex:neutralisation}
$ \xymatrix{
\underset{4}{a_1}\ar@{-{*}}[r] \   &  \underset{3}{a_2}   &  \underset{4}{a_3}   \ar[l]\
}$
\end{example}

\emph{Continuity} requires that the acceptability degree of an argument is a 
continuous function of the initial plausibility. The main motivation for adding this mandatory 
characteristic is to exclude semantics that show chaotic behaviour, where  small differences in 
the initial plausibility leads widely divergent acceptability degrees.  
\begin{feature}[Continuity]
	A semantics \bs satisfies Continuity iff, for any sequence of \wasa\!\!\!\! s
		$\arggraph_n  =\langle {\argset, G, w_n }\rangle $, if
$$\lim_{n\to\infty}w_n = w$$
then for $\arggraph  =\langle {\argset, G, w }\rangle $
$$\lim_{n\to\infty}\accdegr{}{\arggraph_n}{} = \accdegr{}{}{}$$
 \end{feature}

\emph{Interchangeability} requires that arguments with the same acceptability degree may be substituted for each other in attacking and supporting relationships without affecting 
the acceptability degrees of the arguments in the \wasa{}. In other words, for the purpose of calculating 
the acceptability degree of an argument $a$ the identity of the 
supporting and attacking arguments is 
not important, only their acceptability degrees matter.

\begin{feature}[Interchangeability]
	A semantics \bs satisfies Interchangeability iff, for any  any \wasa
  	\defaultarggraph and $\arggraphprime = \langle \argset , G^\prime , w\rangle$, 
	if
	 \begin{itemize}
		\item $a_i, a_j, a_k$ are in the vector \argset, 
		\item $\accdegr{}{}{a_j} = \accdegr{}{}{a_k}$,
		\item \[    G^\prime =
     \left(  \begin{smallmatrix}
       g_{1 1}^\prime &  \cdots & g_{1 n}^\prime \\
       \vdots  & \ddots & \vdots  \\
      g_{n 1}^\prime  & \cdots & g_{n n}^\prime 
      \end{smallmatrix}\right), where 
	  \begin{cases}
	  	g^\prime _{i j} = g_{i k } & \\
	  	g^\prime _{i k} = g_{i j } & \\
		g^\prime _{l m}	= g _{l m}	& \text{otherwise}
	  \end{cases}
\]
	\end{itemize}
then $\accdegr{}{}{} = \accdegr{}{\arggraph ^\prime}{}  $	.

\end{feature}	

Example \ref{ex:interchangeability} illustrates Interchangeability.  The difference between 
\arggraph and \arggraphprime is the direction of attack and support for $a_2$. Since $a_1$ and 
$a_4$  have the same acceptability degree, it follows from Interchangeability that $\accdegr{}{}{} = \accdegr{}{\arggraph ^\prime}{}$.
\begin{example} \label{ex:interchangeability}
	\hspace{1em}
	$\arggraph = $
	\begin{minipage}{0.3 \textwidth}
	$$\xymatrix{
	\underset{0.5}{a_1} \ar[r]\ar[d]   &  \underset{2}{a_2}   \\
	\underset{1}{a_3}   &   \underset{0.5}{a_4}  \ar@{-{*}}[l] \ar@{-{*}}[u]
	}$$
	\end{minipage}	
	\hspace{1em}
		$\arggraphprime = $
	\begin{minipage}{0.3 \textwidth}
	$$\xymatrix{
	\underset{0.5}{a_1}  \ar@{-{*}}[r] \ar[d]   &  \underset{2}{a_2}   \\
	\underset{1}{a_3}   &   \underset{0.5}{a_4}  \ar@{-{*}}[l] \ar[u]
	}$$
	\end{minipage}		
\end{example}

\subsection{Optional Characteristics}\label{sec:optional}
As we pointed out in the discussion of Initial Monotony, the characteristic of 
Proportionality is defined in \cite{amgoudevaluation} in a
 non-standard way. Hence, the question arises whether we can introduce a notion of proportionality 
 in its usual sense? The answer is that proportionality is most likely not a useful concept for 
 an acceptability semantics, since the acceptability degree of an argument depends on two variables: namely, its initial plausibility and the \wasa it is part of. However, instead 
 we can consider \emph{Linearity}. 
 Given a set of arguments and attack and support
relationships between them, Linearity requires that the
acceptability degree of an argument is a linear function of its initial
plausibility.


\begin{feature}[Linearity] \label{feat:linearity}
  A semantics \bs satisfies \emph{Linearity} iff, for any
  \wasa \defaultarggraph  and for any argument $a$ in $ \argset$, there are
  constants $c_1$ and $c_2$ such that for all $w'$ that agree with $w$
  except possibly on $a$,
$$\accdegr{}{\langle {\argset, G, w'}\rangle}{a}=c_1+c_2w'(a)$$
\end{feature}
In contrast to the other characteristics that we discussed in this section, we do not consider 
Linearity a mandatory characteristic of an acceptability semantics. The same is true for Boundedness and Reverse Impact.

A semantics \bs is \emph{bounded} if its acceptability degree space \ads has a maximum and a 
minimum element.

\begin{feature}[Boundedness]
	A semantics \bs satisfies \emph{Boundedness} iff	 \ads is bounded from above 
	and bounded from below. 
\end{feature}

Reverse impact means that the effect of an attack relation
can be supporting or vice versa. For example, assume that $b$ is a discredited argument, which is 
strongly rejected by the audience. If $b$ attacks the argument $a$, then $a$ may actually be 
considered more acceptable by the audience because of the attack. 

\begin{feature}[Reverse impact]\label{f:rev-impact}
	A semantics \bs satisfies \emph{Reverse impact} 
	 iff, for any \wasa 
		\defaultarggraph  
	 and 
	any argument $a$ in $\argset$ there is some argument $b$ in $\argset$ such that:
	\begin{itemize}
		\item  If
		\begin{enumerate}
			\item  $ b \in \attackers{}{a}$ and
			\item $b\not\in \backers{}{a}$ 
		\end{enumerate}  
	then $\accdegr{}{}{a} > \accdegr{}{\arggraph|_{b} }{a} $  or $\accdegr{}{}{a}=\accdegr{}{\arggraph|_{b}}{a}=\maxd$ or $\accdegr{}{}{a}=\accdegr{}{\arggraph|_{b}}{b}=\mind$,
		\item  If  
		\begin{enumerate}
			\item $ b \in \supporters{}{a}$ and
			\item $b\not\in \detractors{}{a}$,
		\end{enumerate}
	then $ \accdegr{}{\arggraph|_{b} }{a} > \accdegr{}{}{a} $ or $\accdegr{}{}{a}=\accdegr{}{\arggraph|_{b}}{a}=\maxd$ or $\accdegr{}{}{a}=\accdegr{}{\arggraph|_{b}}{b}=\mind$.
	\end{itemize}

\end{feature}

\subsection{WASA vs. Support Argumentation Graphs}\label{sec:frenchSection}

In the discussion of Characteristic \ref{feat:neutrality}, we explained our motivation for 
deviating from definitions of the characteristics within \cite{DBLP:conf/ijcai/AmgoudB16, amgoudevaluation} by defining these characteristics by comparison across different \wasa{}.  
We claimed that a cross-graph comparison is applicable to a wider range of examples, but 
we have not discussed the  relationship between the approaches in more detail.  In this 
section we show that our approach is more general. For this purpose we will focus on the definitions of support argumentation graphs in 
\cite{DBLP:conf/ijcai/AmgoudB16}. In a first step we show how the basic definitions in 
\cite{DBLP:conf/ijcai/AmgoudB16} can be represented within our framework by definitions 
\ref{def:bwsa} and \ref{def:boundedAcceptabilitySemantics}. In a second step we show how that the  
 mandatory axioms in \cite{DBLP:conf/ijcai/AmgoudB16}  are entailed by our mandatory characteristics. A third step is relegated to section~\ref{sec:rsig}
below, namely the proof that all our  mandatory axioms hold for the
aggregation-based semantics in \cite{DBLP:conf/ijcai/AmgoudB16}.

\medskip\noindent
We first restrict the general setting to that in
\cite{DBLP:conf/ijcai/AmgoudB16}:
\begin{definition}[Bounded Weighted Support Argumentation Graph]\label{def:bwsa}\ \hfill
		
\noindent A $[0,1]$-Bounded Weighted Support Argumentation Graph (\bwsa) is a \wasa \defaultarggraph such that \begin{itemize}
	\item $g_{ij}\in \{0,1\}$, for any component  $g_{ij}$ of $G$, 
	\item $w$ is a vector in the interval $[0,1]$. 
\end{itemize} 
\end{definition}

Any support argumentation graph in \cite{DBLP:conf/ijcai/AmgoudB16} is a \bwsa. Since there are no 
attacking relationships within \bwsa, it follows that, for any argument $a$ in an \bwsa \arggraph, 
$\attackers{}{a} = \detractors{}{a} = \emptyset$.

We use \bwsa to define a corresponding semantic notion in Definition \ref{def:boundedAcceptabilitySemantics}.  
\begin{definition}[Bounded Acceptability Semantics]\label{def:boundedAcceptabilitySemantics}
A $[0,1]$-bounded acceptability semantics \bs is an acceptability semantics \bs such that $\accdegr{}{}{a} \in [0,1]$, for  any \bwsa \defaultarggraph and any argument $a \in \argset$.
\end{definition}
Note that  any semantics for support argumentation graphs in \cite{DBLP:conf/ijcai/AmgoudB16}  corresponds to a bounded acceptability semantics that is restricted to  \bwsa.

In the following we show that the mandatory axioms in \cite{DBLP:conf/ijcai/AmgoudB16} for bounded acceptability semantics are 
are entailed by the mandatory characteristics for acceptability semantics in section \ref{sec:mandatory}. In some cases this is quite trivial, since the characteristics in section 
\ref{sec:mandatory} are based on \cite{DBLP:conf/ijcai/AmgoudB16}. For this reason we consider 
only two  examples.

The Dummy axiom in \cite{DBLP:conf/ijcai/AmgoudB16} corresponds to our notion of Neutrality and is defined by comparing the acceptability degrees within one argument graph. 
\begin{feature}[Dummy]\label{feat:dummy}
	A bounded acceptability semantics \bs satisfies \emph{Dummy} iff, for any \bwsa and $a, b$ in 
	\argset such that 
	\begin{itemize}
		\item $w(a) = w(b)$,
		\item $\supporters{}{a} = \supporters{}{b} \cup \{ x \} $, such that
		$\accdegr{}{}{x} = \neutrald$, 
	\end{itemize}
	then  $\accdegr{}{}{a} = \accdegr{}{}{b}$.

\end{feature}

\begin{theorem}\label{thm:Dummy}
Any bounded acceptability semantics \bs that satisfies Neutrality and Parent Monotony, satisfies 
 Dummy. 
\end{theorem}
Note that boundedness or the restriction to $[0,1]$ is not needed for
Thm.~\ref{thm:Dummy}; all that is needed is that there are no attacks.

\emph{Stickiness}  (which is called boundedness in  \cite{DBLP:conf/ijcai/AmgoudB16}) expresses 
that if  argument $b$ has an acceptability degree of 1 and argument $a$ has some stronger support than $a$ has also an acceptability degree of 1. 

\begin{feature}[Stickiness]\label{feat:stickiness}
	A bounded acceptability  semantics \bs satisfies \emph{Stickiness} iff, for any \bwsa \defaultarggraph and $a, b$ in 
	\argset such that 
	\begin{itemize}
		\item $w(a) = w(b)$,
		\item $\supporters{}{a} \backslash \supporters{}{b} = \{ x \} $,
		\item $\supporters{}{b} \backslash \supporters{}{a} = \{ y \} $,
		\item $\accdegr{}{}{x} > \accdegr{}{}{y}$, 
	\end{itemize}
	if $\accdegr{}{}{b} = 1$, then $\accdegr{}{}{a} = 1$.
	
\end{feature}
\begin{theorem} Any bounded acceptability semantics \bs that satisfies Parent Monotony, Independence, Interchangeability, Anonymity  and Conservativity, satisfies Stickiness. \label{thm:stickiness} 
\end{theorem}

Analog theorems may be formulated for the other axioms in \cite{DBLP:conf/ijcai/AmgoudB16}:

\begin{theorem}\label{thm:French-characteristics}
 Any bounded acceptability semantics \bs that satisfies the mandatory characteristics in \ref{sec:mandatory}, satisfies 
 Anonymity, 
 Independence, 
 Non-Dilution, 
 Monotony, 
 Equivalence, 
 Dummy, 
 Minimality, 
 Strengthening, 
 Strengthening Soundness, 
 Coherence, 
 Counting, 
 Boundedness, and 
 Reinforcement as defined in \cite{DBLP:conf/ijcai/AmgoudB16}. 
\end{theorem}

Thm.~\ref{thm:French-characteristics}, together with
Thm.~\ref{thm:French-aggr-properties} below, shows that our axiomatic
approach is more general than that in
\cite{DBLP:conf/ijcai/AmgoudB16}.

\section{Direct Aggregation Semantics}	\label{sec:semantics}	

\renewcommand{\accdegr}[3]{\ensuremath{\texttt{Deg}^{\ifthenelse{\equal{#1}{}}{ad}{#1}}%
_{\ifthenelse{\equal{#2}{}}{\arggraph}{#2} }}(#3)\xspace}

The main intuition behind the aggregation semantics is that the strength of an argument in an argumentation
is based 
on its initial plausibility and it is strengthened by supports and weakened by attacks. 
If we understand "strengthening" as addition and "weakening" as subtraction, 
we get to a model where the acceptability degree of an argument is calculated by 
adding to its initial plausibility the sum of the acceptability degrees of its supporters and 
subtracting the sum of the acceptability degree of its attackers. 

One additional intuition behind 
the semantics is 
that the influence of arguments is the strongest on arguments that they directly attack or support 
and increasingly weaker on arguments that are only indirectly connected to them.   

\begin{example}\label{ex:dampening}
\begin{minipage}{0.3 \textwidth}
$$\xymatrix{
\underset{1}{a_3} \ar[r]   &  \underset{1}{a_2} \ar[r] &  \underset{1}{a_1}\\
\underset{1}{a_4} \ar[u]   &  &
}$$
\end{minipage}	
\begin{minipage}{0.5 \textwidth}	
$$\xymatrix{
\underset{1}{a_3^\prime} \ar[r]   &  \underset{1}{a_2^\prime} \ar[r] &  \underset{1}{a_1^\prime}\\
& & \underset{1}{a_4^\prime} \ar[u]   &  
}$$
\end{minipage}
\end{example}
The difference can be illustrated by Example \ref{ex:dampening}: both $a_1$ and $a_1^\prime$ are 
supported by 3 different arguments. However, while $a_4$ provides only indirect 
support for $a_1$, $a_4^\prime$ supports $a_1^\prime$ directly. For this reason, the acceptability 
degree of $a_1^\prime$ should be larger than the acceptability degree of $a_1$. We achieve this 
effect by introducing a so-called "dampening factor" $d\geq 1$ that mitigates the effect of arguments 
along the paths of a \wasa. 

\subsection{Definition of Direct Aggregation Semantics}

These two intuitions are formalised  by Definitions \ref{def:aggsemf} and  \ref{def:aggsem}. 

\begin{definition}\label{def:aggsemf}
Let \defaultarggraph be a \wasa. Let the damping factor $d$ be such that $d \geq 1$. 
 For $i\in\mathbb{N}$, let $\fsemdir_i$ be a function\footnote{$\fsemdir$ also depends on $G$ and $d$, which we omit here for readability. Since $\fsemdir$ is only used locally in this definition and the next one, $G$ and $d$ are clear from context.} from $\argset$ to 
$\mathbb{R}$ such that
for any 
$a\in \argset$, for $i \in \{0, 1, 2, \ldots  \} $, if $i = 0$, then $\fsemdir_i(a) = w(a)$, otherwise

		\[ \fsemdir_i(a) = w(a) + \frac{1}{d} \times \left (\sum_{b\in \supporters{}{a}} \fsemdir_{i-1}(b) - 
					\sum_{c\in \attackers{}{a}} \fsemdir_{i-1}(c)\right )						
	\]
or shorter in matrix notation	
		\[ \fsemdir_i = w + \frac{1}{d}G \fsemdir_{i-1}	
	\]
\end{definition}

\begin{definition}\label{def:propagation}
We call $G$ the \emph{incidence matrix} and 
$$\Pr^{G,d}=\sum_{i=0}^\infty \ (\frac{1}{d}G)^i$$ 
the \emph{propagation matrix}.
\end{definition}

\begin{definition}\label{def:aggsem}
The \emph{direct aggregation semantics} is a function $s^d$ transforming any \wasa 
\defaultarggraph into a weighting on \argset such that for any $a\in \argset $

	\[
	   \accdegr{\semdir}{\arggraph,d}{a} =  \lim_{i\rightarrow \infty} \ \fsemdir_{i}(a)
	\]
that is
	\[
	   \accdegrvec{\semdir}{\arggraph,d} =  \lim_{i\rightarrow \infty} \ \fsemdir_i = \sum_{i=0}^\infty \ (\frac{1}{d}G)^iw
	\]
or short
	\[
	   \accdegrvec{\semdir}{\arggraph,d} = \Pr^{G,d} w
	\]

\begin{center}
 (degree = propagation matrix $\times$ initial plausibility)
\end{center}
\end{definition}

This means that the degree can be computed from the initial
plausibility by a fixed linear transformation given by the propagation
matrix.

Let $\ind(G)$ be the maximal indegree, i.e.\ the maximum number of edges leading into an argument in \arggraph.  For $d\leq \ind(G)$, the direct aggregation semantics is not well-defined
in general. Example \ref{ex:direct-aggregation-convergence} provides an example where $\fsemdir$ 
does not converge. 

\begin{example}\label{ex:direct-aggregation-convergence}

Assume $d = \ind(G)$ and 
consider an argument with an initial plausibility $1$ that attacks itself. In this case $d = \ind(G) = 1$ and 

\begin{center}
$
\arggraph  =\langle
\matr{a  }, \matr{ -1}, \matr{1  }\rangle 
$
\qquad or, graphically: \qquad
$\arggraph  =\xymatrix{
  \underset{1}{a}  \ar@(r,u)@{-{*}}
}$	
\end{center}
\noindent In this case $\fsemdir_{2i}(a)=1$ and $\fsemdir_{2i+1}(a)=0$.
\end{example}

\begin{theorem}\label{thm:direct-aggregation-convergence}
For $d>\ind(G)$, the direct aggregation semantics is well-defined
and converges to $(I-\frac{1}{d}G)^{-1}w$.
\end{theorem}
For the proof, we need the following from \cite{Horn:1985:MA:5509}, Corolloray 5.6.16:
\begin{fact}\label{fact:matrix-convergence}
If $\norm{\_}$ is a matrix norm, and if $\norm{A}<1$, then $I-A$ is invertible
and
$$(I-A)^{-1}=\sum_{i=0}^\infty  A^i$$
\end{fact}
\begin{proof}{ (of Fact~\ref{fact:matrix-convergence})}
$\sum_{i=0}^\infty A^i$ converges because its norm does. But then
$(I-A)(\sum_{i=0}^\infty A^i) = \sum_{i=0}^\infty  A^i - \sum_{i=1}^\infty  A^i = A^0 = I$.
\end{proof}

\begin{proof}{ (of Thm.~\ref{thm:direct-aggregation-convergence})}
We need to show that
$$\Pr^{G,d}=\sum_{i=0}^\infty \ (\frac{1}{d}G)^i$$
converges. 
Note that the maximum row sum norm $\rownorm{\_}$ defined by
$$\rownorm{G}=\max_{i=1,\ldots,n}\sum_{j=1,\ldots,n}|g_{ij}|$$
coincides with the maximal indegree, i.e.\ $\rownorm{G}=\ind(G)$.
Hence, we have
$$\rownorm{\frac{1}{d}G}\leq \frac{\ind(G)}{d} < 1$$
By Fact~\ref{fact:matrix-convergence}, this implies that
$I-\frac{1}{d}G$ is invertible and $\Pr^{G,d}=(I-\frac{1}{d}G)^{-1}$, thus
$$\accdegrvec{\semdir}{\arggraph,d}=\Pr^{G,d} w=\sum_{i=0}^\infty \ (\frac{1}{d}G)^iw=(I-\frac{1}{d}G)^{-1}w\vspace*{-0.4cm}$$
\end{proof}

\subsection{Application of Direct Aggregation Semantics}
To illustrate the use of the direct aggregation semantics we revisit an example from the literature 
 \cite{cayrol2013, BGTV10b}. Assume it is the last weekend of the football season and 
there is a close title race between Liverpool and Manchester United. Liverpool will win the title if it either wins its last game or Manchester does not win its last game. The question is: Will Liverpool win the Premiere League? (lpl).  There are two arguments supporting Liverpool's title ambitions:  Liverpool will win, because Manchester United will not win its last game against Manchester City (mnw). And Liverpool will win the tile, because it will win its match against Arsenal (wlm). However, there is a counterargument: Liverpool will not win against Arsenal, because Liverpool's best player is injured (bpi). 

This situation may be represented similarly as in  \cite{cayrol2013}:
\begin{example}
\label{ex:liverpool}
\hspace{3em}
	\begin{minipage}{0.5\textwidth}
	$$\arggraph =  \xymatrix{
	\underset{8}{mnw} \ar[r]   &  \underset{0}{lpl}   &     \underset{5}{wlm}\ar[l] &   \underset{2}{bpi}\ar@{-{*}}[l]  \\
	}$$
	\end{minipage}	
	
\end{example}
However, since \wasa{} are weighted, they allow us to not just represent the relationships between the arguments, but also the initial plausibility of the arguments involved. Since the question whether Liverpool will win is the topic 
under discussion we assign it an initial value of 0 (the neutral value). 
Let us assume that we polled a panel of eight experts on the plausibility of the arguments. 
5 experts believe Liverpool will win its game. Let's further assume that 
3 experts believe that Manchester City will win and 5 experts believe in a draw in the Manchester derby. Hence, ``Manchester will not win'' (mnw) is assigned a value of 8. Further, only 2 experts believe that the loss of Liverpool's star player will have a significant impact on the game.

Let's assume a dampening factor of 2 for the remainder of this subsection. In this case it follows that in Example \ref{ex:liverpool} $\accdegr{\semdir}{}{bpi} = 2$, 
$\accdegr{\semdir}{}{wlm} = 4$, 
$\accdegr{\semdir}{}{mnw} = 8$,
$\accdegr{\semdir}{}{lpl} = 6$. Thus, given the argumentation  \arggraph in Example \ref{ex:liverpool}, Liverpool fans should be quite optimistic. 

This can be seen in more detail as follows. With $\argset=\matr{mnw\\lpl\\wlm\\bpi}$,
the matrix of the graph is
$$G =\matr{0&0&0&0\\1&0&1&0\\0&0&0&-1\\0&0&0&0}$$
and using damping factor $d=2$, the propagation matrix is
$$\Pr^{G,2} = (I-\frac12G)^{-1} = \matr{1&0&0&0\\0.5&1&0.5&-0.25\\0&0&1&-0.5\\0&0&0&1}$$
The computed acceptability degrees are $\Pr^{G,2}\matr{8\\0\\5\\2}=\matr{8\\6\\4\\2}$.

Note that in the example the first argument (mnw) may be considered as a summary of two different arguments: an argument that Manchester City wins (mcw) or that the Manchester derby is a draw (mdd). If we split these arguments up, we get the \wasa \arggraphprime in Example \ref{ex:liverpool2}. 

\begin{example}
\label{ex:liverpool2}
\hspace{3em}
\arggraphprime =
	\begin{minipage}{0.5\textwidth}
	$$  \xymatrix{
	\underset{5}{mdd} \ar[r]   &  \underset{0}{lpl}   &     \underset{5}{wlm}\ar[l] &   \underset{2}{bpi}\ar@{-{*}}[l]  \\
	& \underset{3}{mcw} \ar[u] &  &  
	}$$
	\end{minipage}	
\end{example}

With $\argset=\matr{mdd\\lpl\\wlm\\bpi\\mcw}$, the matrix of the graph is
$$G' =\matr{0&0&0&0&0\\1&0&1&0&1\\0&0&0&-1&0\\0&0&0&0&0\\0&0&0&0&0}$$
and using damping factor $d=2$, the propagation matrix is
$$\Pr^{G',2} = (I-\frac12G')^{-1} = \matr{1&0&0&0&0\\0.5&1&0.5&-0.25&0.5\\0&0&1&-0.5&0\\0&0&0&1&0\\0&0&0&0&1}$$
The computed acceptability degrees are $\Pr^{G',2}\matr{5\\0\\5\\2\\3}
=\matr{5\\6\\4\\2\\3}$.

Because the effects of supporting arguments in the direct aggregation semantics is additive, it does not matter whether one chooses the representation in Example \ref{ex:liverpool} or in Example \ref{ex:liverpool2}. The acceptability degree of ``Liverpool wins the Premiere League" (lpl) does not change.  

A third alternative of representing the situation is to consider Manchester United's perspective (see Example \ref{ex:liverpool3}). ``Manchester United wins the Premiere League'' (mpl) is under the attack by (mdd), (mcw), and (wlm). As result, $\accdegr{\semdir}{\arggraph^{\prime\prime}}{mpl} = -6$.


\begin{example}
\label{ex:liverpool3}
\hspace{3em}
	\begin{minipage}{0.5\textwidth}
	$$\arggraph^{\prime\prime} =  \xymatrix{
	\underset{5}{mdd} \ar@{-{*}}[r]   &  \underset{0}{mpl}   &     \underset{5}{wlm}\ar@{-{*}}[l] &   \underset{2}{bpi}\ar@{-{*}}[l]  \\
	& \underset{3}{mcw} \ar@{-{*}}[u] &  &  
	}$$
	\end{minipage}	
	
\end{example}

Another example concern two pupils Alice and Bob that accuse each
other of lying about a certain circumstance. The teachers Miller
and Smith support the views of Alice and Bob, respectively, because they know
their pupils well.  Of course, both pupils also accuse the other teacher to be
wrong, while the teachers are wise enough not to attack each others'
views, but even support them because they have known each other for a
long time.  From the files, the director derives some judgement about
credibilities:

\begin{example}
\label{ex:school}
\hspace{3em}
$ \overline{\arggraph} =$
	\begin{minipage}{0.4\textwidth}
	$$  \xymatrix{
	\underset{6}{Miller} \ar[d]\ar@{<->}[r]   &  \underset{4}{Smith}\ar[d] \\
	\underset{1}{Alice} \ar@{-{*}}[ur]\ar@{{*}-{*}}[r] &
	\underset{1.5}{Bob} \ar@{-{*}}[ul]
	}$$
	\end{minipage}	
	
\end{example}
With $\overline{\argset}=\matr{Miller\\Smith\\Alice\\Bob}$, the matrix of the graph is
$$\overline{G} = \matr{0&1&0&-1\\1&0&-1&0\\1&0&0&-1\\0&1&-1&0}$$
and using damping factor $d=3$, the propagation matrix is
$$\Pr^{\overline{G},3} = (I-\frac13\overline{G})^{-1} = \frac{1}{21}\matr{23&5&1&-8\\5&23&-8&1\\8&-1&25&-11\\-1&8&-11&25}$$
The computed acceptability degrees are $\Pr^{\overline{G},3}\matr{6\\4\\1\\1.5}=\matr{7\\5.5\\2.5\\2.5}$.
This means that Alice's view is valued equally with Bob's, even
though her initial assessment has been slightly worse than Bob's.
Also note that the all values have increased, because  support
has been in general stronger than attack.

\subsection{Properties of Direct Aggregation Semantics}

\begin{theorem}\label{thm:direct-aggregation-equation}
$D=\accdegrvec{\semdir}{\arggraph,d}$ is the unique solution of the equation
$$D = w+\frac{1}{d}GD$$
\centering(degree = initial plausibility + damped influence)
\end{theorem}
\begin{proof}{}
  The equation can be rewritten as $(I-\frac{1}{d}G)D=w$, which shows that
  $D=(I-\frac{1}{d}G)^{-1}w=\accdegrvec{\semdir}{\arggraph,d}$ is a solution. Moreover, if $D =
  w+\frac{1}{d}GD$, by successively unfolding, we get
$$D=\sum_{i=0}^\infty \ (\frac{1}{d}G)^iw=\accdegrvec{\semdir}{\arggraph,d}\vspace*{-0.8cm}
$$
\end{proof}

A node $a$ in an argumentation graph is called \emph{circular}, if
there is a non-empty path of parent relations (which can be attack or
support relations) that starts and ends in $a$. A node is called
\emph{hereditarily circular}, if one of its backers or detractors
is circular.

\begin{theorem}\label{thm:direct-aggregation-convergence2}
  Convergence as in Thm.~\ref{thm:direct-aggregation-convergence}
  holds already if $d$ is greater than the maximum indegree for the
  subgraph induced by the hereditarily circular nodes; or $d>0$ if
  there are no hereditarily circular nodes, i.e. if the argumentation
  graph is acyclic.

\end{theorem}
\begin{proof}{}
  If $a_i$ is not hereditarily circular and $k\geq n$, the $i$-th row
  of $G^k$ consists of zeros only. Hence, for convergence, it suffices
  to consider the hereditarily circular nodes only.
\end{proof}

In the sequel, we will examine direct aggregation semantics w.r.t.\
the characteristics introduced in Sect.~\ref{sec:characteristics}.
Generally, we have two options of choosing the damping factor $d$:
\begin{enumerate}
\item $d$ is chosen such that $d>\ind(G)$. To be definite, let us
  chose $d=\ind(G)+1$. This means that $d$ depends on the argument
  matrix (graph).
\item $d$ is chosen globally $d$, independently of the matrix (graph).
  This means that for each $d$, we get a separate semantics, which is
  guaranteed to be defined only for graphs $G$ with $\ind(G)<d$.
\end{enumerate}
We will indicate these cases.

\begin{example}
If  $d=\ind(G)+1$ (or generally, $d$ does depend on $\ind(G)$),
direct aggregation semantics does not satisfy  Independence.
Consider 
$\arggraph  =\langle
\matr{a },
\matr{
1 
},
\matr{1}\rangle $
and
$\arggraph'  =\langle
\matr{a  \\b  \\c},
\matr{
-1 & 0 & 0 \\
0 & -1 & -1 \\
0 & -1 & -1 \\
},
\matr{1  \\1  \\1}\rangle$.

Then $\accdegrvec{\semdir}{\arggraph,2}=\matr{\frac{2}{3}}$,
but
$\accdegrvec{\semdir}{\arggraph',3}=\matr{\frac{3}{4} & \frac{3}{5} & \frac{3}{5} }^T$.
\end{example}

\begin{theorem}
  If $d$ is chosen globally (which implies that only \mbox{\wasa\!\!s}
  of maximal indegree $<d$ are considered), direct aggregation semantics
  satisfies  Independence.
\end{theorem}
\begin{proof}{}
  $\accdegr{\semdir}{\arggraph,d}{a} = \left(\sum_{i=0}^\infty \
    (\frac{1}{d}G)^iw\right)(a)\footnote {Given a vector $v$, we use the
    notation $v(a_j)$ for the value $v_j$. This is particularly
    convenient if we do not have the index $j$ at hand.}  =
  \left(\sum_{i=0}^\infty \ (\frac{1}{d}\matr{G^i & 0 \\0 & (G^\prime)^i})
    \matr{w \\w'}\right)(a) =$ \\ 
	\hspace*{\fill}$ \left(\sum_{i=0}^\infty \ (\frac{1}{d}\matr{G & 0
      \\0 & G^\prime})^iw^{\dagger}\right)(a) =\left(\sum_{i=0}^\infty
    \ (\frac{1}{d}G^{\dagger})^iw^{\dagger}\right)(a)
  =\accdegr{\semdir}{\arggraph\oplus \arggraph ^\prime,d}{a}$.
\end{proof}

Our semantics enjoys all other desirable properties (the proof
has been relegated to the appendix):
\begin{theorem}\label{thm:direct-aggr-properties}
  Direct aggregation semantics satisfies Anonymity, Equivalence,
  Directionality, Conservativity, Initial Monotony, Neutrality,
  Parent Monotony, Impact, Reinforcement, Causality, Neutralisation and Continuity, Linearity and Reverse impact.
\end{theorem}

\subsection{Behaviour of the Propagation Matrix}

\newcommand{\di}{d} We will now examine the behaviour of direct
aggregation semantics for specific cases. As noted above, the degree
is computed from the initial plausibility by multiplication with the
propagation matrix. Hence, it suffices to examine propagation
matrices, which are independent of the initial plausibility. Indeed,
propagation matrices can be computed easily by matrix inversion, see
Thm.~\ref{thm:direct-aggregation-convergence}.  We list propagation
matrices for a number of small argument graphs, in terms of the
damping factor $d$. We begin with acyclic graphs. An
isolated argument will be not affected at all:

$$\begin{array}{cllcl}
G & \Pr^{G,d}\\
a & (1)
\end{array}$$

Edges in the argument graph propagate (positive or negative) influence
with strength $\frac{1}{\di}$:
$$\begin{array}{cllcc}
G & \Pr^{G,d} && G & \Pr^{G,d}\\[0.5em]
\xymatrix{a \ar[r]& b} & \matr{1 & 0\\\frac{1}{\di} & 1} &&
\xymatrix{a \ar@{-{*}}[r]& b} & \matr{1 & 0\\-\frac{1}{\di} & 1}
\end{array}$$

Influence of arguments can also propagated along $k$ edges,
resulting in a factor $\frac{1}{\di^k}$. For $k=2$, we get:

$$\begin{array}{cclcc}
G & \Pr^{G,d} && G & \Pr^{G,d}\\[0.5em]
\xymatrixcolsep{15pt}
\xymatrix{a \ar[r]& b\ar[r]&c} & \frac{1}{\di^2}\matr{\di^2 & 0 & 0\\\di & \di^2 & 0\\1 & \di & \di^2} &&
\xymatrixcolsep{15pt}
\xymatrix{a \ar@{-{*}}[r]& b\ar@{-{*}}[r]&c} & \frac{1}{\di^2}\matr{\di^2 & 0 & 0\\-\di & \di^2 & 0\\1 & -\di & \di^2}\\[1.5em] 
\xymatrixcolsep{15pt}
\xymatrix{a \ar[r]& b\ar@{-{*}}[r]&c} & \frac{1}{\di^2}\matr{\di^2 & 0 & 0\\\di & \di^2 & 0\\-1 & -\di & \di^2} &&
\xymatrixcolsep{15pt}
\xymatrix{a \ar@{-{*}}[r]& b\ar[r]&c} & \frac{1}{\di^2}\matr{\di^2 & 0 & 0\\-\di & \di^2 & 0\\-1 & \di & \di^2}
\end{array}$$

Now to cyclic graphs. The simplest case are self-support and \ self-attack,
resulting respectively in a slight strengthening and weakening  of the weight:\medskip

$$\begin{array}{cllllcll}
G && \Pr^{G,d} && G && \Pr^{G,d}\\[1.3em]
\xymatrix{a\ar@(r,u)} && \matr{\frac{\di}{\di-1}}&&
\xymatrix{a\ar@(r,u)@{-{*}}} && \matr{\frac{\di}{\di+1}}
\end{array}$$

Since a chain of two attacks is a support, both mutual attack and
support lead to a slight self-support as well, while an attack-support
pair (called a vicious circle in \cite{Betz05}) leads to a slight
self-attack of both arguments.  We also show the combination with explicits
self-supports and self-attacks:
$$\begin{array}{cclcc}
G & \Pr^{G,d} && G & \Pr^{G,d}\\[0.5em]
\xymatrix{a \ar@{<->}[r]& b} & \frac{1}{\di^2-1}
   \matr{\di^2 & \di\\\di & \di^2} &&
\xymatrix{a \ar@{{*}-{*}}[r]& b} & \frac{1}{\di^2-1}
   \matr{\di^2 & -\di\\-\di & \di^2}\\[1.5em]   
\xymatrix{a \ar[r]<0.5ex>& b\ar@{-{*}}[l]<0.5ex>} & \frac{1}{\di^2+1}
   \matr{\di^2 & -\di\\\di & \di^2}&&
\xymatrix{a \ar[r]<0.5ex>& b\ar@(r,u)\ar@{-{*}}[l]<0.5ex>} & \frac{1}{\di^2-r+1}
   \matr{\di^2-\di & -\di\\\di & \di^2}\\[1.5em]   
\xymatrix{a \ar@{<->}[r] & b\ar@(r,u)} & \frac{1}{\di^2-\di-1}
   \matr{\di^2-\di & \di\\\di & \di^2}&&
\xymatrix{a \ar@{{*}-{*}}[r]& b\ar@(r,u)@{-{*}}} & \frac{1}{\di^2+\di-1}
   \matr{\di^2+\di & -\di\\-\di & \di^2} \\[1.5em]
\xymatrix{a\ar@(l,u) \ar@{<->}[r] & b\ar@(r,u)} & \frac{1}{\di-2}
   \matr{\di-1 & 1\\1 & \di-1} &&
\xymatrix{a \ar@(l,u)@{-{*}}\ar@{{*}-{*}}[r]& b\ar@(r,u)@{-{*}}} & 
   \frac{1}{\di+2}\matr{\di+1 & -1\\-1 & \di+1}
\end{array}$$

For all arguments in a cycle of length three, we get a slight
self-support (left hand-side, non-vicious circles) or self-attack of
the argument (right hand-side, vicious circle in the sense of
\cite{Betz05}):
$$\begin{array}{cccc}
\xymatrixcolsep{15pt}
G & \Pr^{G,d} & G & \Pr^{G,d}\\[0.5em]
\xymatrix{a \ar[rr]&& b\ar[dl]\\&c\ar[ul]} & \frac{1}{\di^3-1}\matr{\di^3 & \di & \di^2\\\di^2 & \di^3 & \di\\\di & \di^2 & \di^3} &
\xymatrixcolsep{15pt}
\xymatrix{a \ar@{-{*}}[rr]&& b\ar@{-{*}}[dl]\\&c\ar@{-{*}}[ul]} & \frac{1}{\di^3+1}\matr{\di^3 & \di & -\di^2\\-\di^2 & \di^3 & \di\\\di & -\di^2 & \di^3} \\[1.5em] 
\xymatrixcolsep{15pt}
\xymatrix{a \ar@{-{*}}[rr]&& b\ar@{-{*}}[dl]\\&c\ar[ul]} & \frac{1}{\di^3-1}\matr{\di^3 & -\di & \di^2\\-\di^2 & \di^3 & -\di\\\di & -\di^2 & \di^3} &
\xymatrixcolsep{15pt}
\xymatrix{a \ar[rr]&& b\ar[dl]\\&c\ar@{-{*}}[ul]} & \frac{1}{\di^3+1}\matr{\di^3 & -\di & -\di^2\\\di^2 & \di^3 & -\di\\\di & \di^2 & \di^3} \\[1.5em] 
\end{array}$$


Finally, we show various cycles of length four, all of which
show a slight self-support. The numbers $2$ and $-2$ in the
last three examples are caused by the fact that between diagonally
opposite arguments, there are always two distinct paths:
$$\begin{array}{cccc}
G & \Pr^{G,d} & G & \Pr^{G,d}\\[0.5em]
\xymatrixcolsep{15pt}
\xymatrixrowsep{15pt}
\xymatrix{a \ar[r]& b\ar[d]\\d\ar[u]&c\ar[l]} & \frac{1}{\di^4-1}\matr{\di^4 & \di & \di^2 & \di^3\\\di^3 & \di^4 & \di & \di^2\\\di^2 & \di^3 & \di^4& \di\\\di & \di^2 & \di^3 & \di^4} &
\xymatrixcolsep{15pt}
\xymatrixrowsep{15pt}
\xymatrix{a \ar@{-{*}}[r]& b\ar@{-{*}}[d]\\d\ar@{-{*}}[u]&c\ar@{-{*}}[l]} & \frac{1}{\di^4-1}\matr{\di^4 & -\di & \di^2 & -\di^3\\-\di^3 & \di^4 & -\di & \di^2\\\di^2 & -\di^3 & \di^4& -\di\\-\di & \di^2 & -\di^3 & \di^4} \\[1.5em] 
\xymatrixcolsep{15pt}
\xymatrixrowsep{15pt}
\xymatrix{a \ar@{-{*}}[r]& b\ar[d]\\d\ar[u]&c\ar@{-{*}}[l]} & \frac{1}{\di^4-1}\matr{\di^4 & -\di & -\di^2 & \di^3\\-\di^3 & \di^4 & \di & -\di^2\\-\di^2 & \di^3 & \di^4& -\di\\\di & -\di^2 & -\di^3 & \di^4} &
\xymatrixcolsep{15pt}
\xymatrixrowsep{15pt}
\xymatrix{a \ar@{-{*}}[r]& b\ar@{-{*}}[d]\\d\ar[u]&c\ar[l]} & \frac{1}{\di^4-1}\matr{\di^4 & -\di & \di^2 & \di^3\\-\di^3 & \di^4 & -\di & -\di^2\\\di^2 & -\di^3 & \di^4& \di\\\di & -\di^2 & \di^3 & \di^4} \\[1.5em] 
\xymatrixcolsep{15pt}
\xymatrixrowsep{15pt}
\xymatrix{a \ar@{<->}[r]& b\ar@{<->}[d]\\d\ar@{<->}[u]&c\ar@{<->}[l]} & \frac{1}{\di^2-\di}\matr{\di^2-2 & \di & 2 & \di\\\di & \di^2-2 & \di & 2\\2 & \di & \di^2-2& \di\\\di & 2 & \di & \di^2-2} &
\xymatrixcolsep{15pt}
\xymatrixrowsep{15pt}
\xymatrix{a \ar@{{*}-{*}}[r]& b\ar@{{*}-{*}}[d]\\d\ar@{{*}-{*}}[u]&c\ar@{{*}-{*}}[l]} & \frac{1}{\di^2-\di}\matr{\di^2-2 & -\di & 2 & -\di\\-\di & \di^2-2 & -\di & 2\\2 & -\di & \di^2-2& -\di\\-\di & 2 & -\di & \di^2-2} \\[1.5em] 
\xymatrixcolsep{15pt}
\xymatrixrowsep{15pt}
\xymatrix{a \ar@{<->}[r]& b\ar@{{*}-{*}}[d]\\d\ar@{{*}-{*}}[u]&c\ar@{<->}[l]} & \frac{1}{\di^2-\di}\matr{\di^2-2 & \di & -2 & -\di\\\di & \di^2-2 & -\di & -2\\-2 & -\di & \di^2-2& \di\\-\di & -2 & \di & \di^2-2} &
\end{array}$$

\subsection{Sigmoid Direct Aggregation Semantics}
\label{sec:sigmoid-direct}

The direct aggregation semantics allows for the real numbers as values for the initial plausibility 
and for the acceptability degree. As discussed on page \pageref{br:sigmoiddirect}, this approach 
deviates from other approaches like \cite{DBLP:conf/ijcai/AmgoudB16, amgoudevaluation}, where 
only a subset of $\mathbb{R}$ is considered. 
In this subsection we illustrate  how the real-valued direct aggregation semantics can be
cast into the more traditional framework of weighting in $[0,1]$.
Actually, for technical reasons, we restrict ourselves to the
interval $(0,1)$. A similar approach could be chosen for other intervals.

To constrain our semantics to $(0,1)$, we need a sigmoid function, that is, a bijection $\sigma:\mathbb{R}\to(0,1)$
that is continuous and strictly increasing such that $\sigma(0)=0.5$.
For example, the logistic function 
$$\sigma(x)=\frac{1}{1+e^{-x}}$$
the suitably normalised arcus tangens function
$$\sigma(x)=\frac{\arctan(x)}{\pi}+\frac{1}{2}$$
or the simple fraction
$$\sigma(x)=\frac{1+|x|+x}{2(1+|x|)}$$
will do. 

\begin{center}
\includegraphics[width=0.6\textwidth]{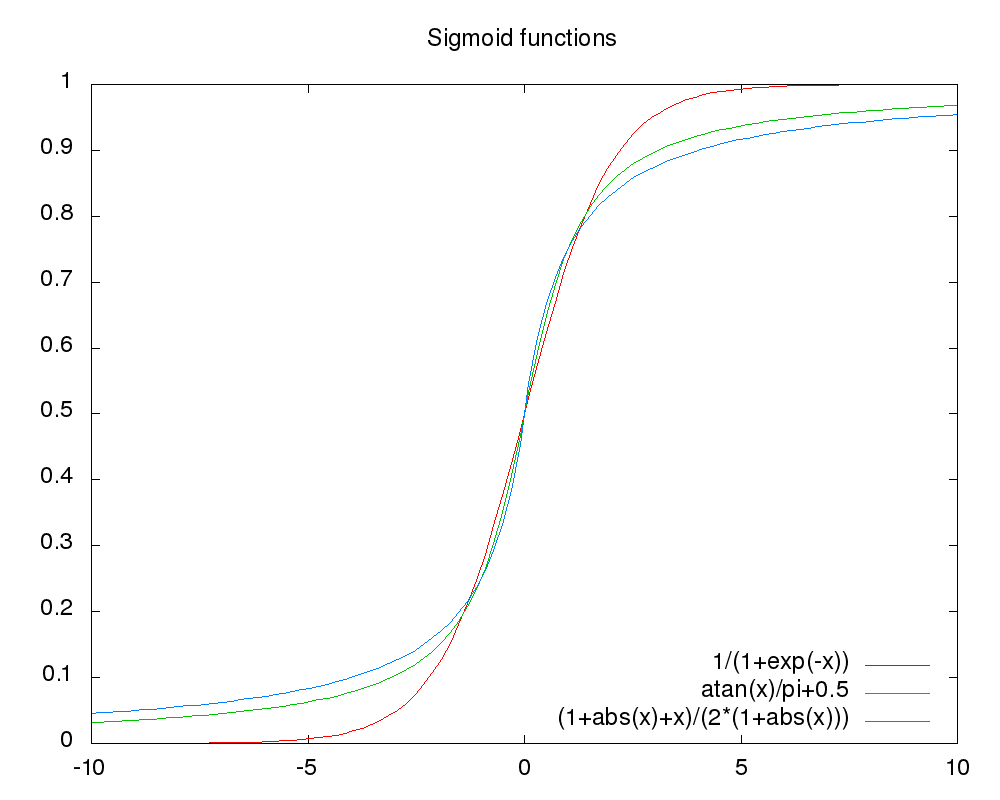}
\end{center}

Being such prepared, we now can define  
\begin{definition}\label{def:saggsemf}
Let \defaultarggraph be a \wasa such that $w:\mathcal{A}\to(0,1)$. 
Let the damping factor $d$ be such that $d \geq 1$. 
The sigmoid direct aggregation semantics is defined as

\[ \accdegrvec{\semsdir,d}{} =  
   \sigma\left(\Pr^{G,d}\sigma^{-1}(w)\right)
	\]
\end{definition}
We have a fixed-point theorem similar to Thm.~\ref{thm:direct-aggregation-equation}:
\begin{theorem}\label{thm:sigmoid-direct-aggregation-equation}
$D=\accdegrvec{\semsdir,d}{}$ is the unique solution of the equation
$$D = \sigma(\sigma^{-1}(w)+\frac{1}{d}G\sigma^{-1}(D))$$
\end{theorem}
\begin{proof}{}
Rewrite the equation to
\begin{eqnarray}
\sigma^{-1}(D) = \sigma^{-1}(w)+\frac{1}{d}G\sigma^{-1}(D)
\label{eq:sigmoid-fixpoint}
\end{eqnarray}
and use Thm.~\ref{thm:direct-aggregation-equation}.
\end{proof}

In the appendix, we prove the desirable properties of our semantics,
with the neutral value taken to be $\frac{1}{2}$ instead of $0$:
\begin{theorem}\label{thm:sigmoid-direct-aggr-properties}
  Using $\neutrald=\frac{1}{2}$,
  sigmoid direct aggregation semantics satisfies Anonymity,
  Equivalence, Directionality, Conservativity, Initial Monotony,
  Parent Monotony, Neutrality, Impact,
  Reinforcement, Causality, Neutralisation, Continuity
  and Reverse impact.
  Independence is satisfied if $d$ is globally fixed.
  Since sigmoid functions are nonlinear, Linearity cannot
  be satisfied; instead only the weaker Initial Monotony holds.
\end{theorem}

It is straightforward to modify sigmoid direct aggregation semantics
such that it works with the interval $(-1,1)$\footnote{Also other
  intervals are possible, but seem less natural.} instead of $(0,1)$.
This might even be considered as more natural, because then $0$ is the
neutral value, and not $\frac{1}{2}$. We here have chosen $(0,1)$ because this
interval is used more frequently in the literature\footnote{More precisely, the
literature used the \emph{closed} interval $[0,1]$.}, and $0$ and $1$
can roughly be thought of as false and true.  

\section{Comparison to previous work}\label{sec:comparison}

Generalising the classical work by Dung \cite{DBLP:journals/ai/Dung95}
to rank-based argumentation, Amgoud et al.\ \cite{amgoudevaluation}
have introduced weighted argumentation graphs. Moreover, they restrict
weightings to $[0,1]$. Thus they consider a non-empty finite set
$\mathcal{A}$ of arguments, a weighting $w:\mathcal{A}\to[0,1]$, and
an attack relation $\attack \subseteq \argset \times \argset$.  Amgoud
et al.\ have also considered support argumentation graphs
\cite{DBLP:conf/ijcai/AmgoudB16}, which are similar, except that
$\attack$ is replaced by a support relation $\support \subseteq \argset
\times \argset$.

It is straightforward to organise $\mathcal{A}$ into a vector.
Then $\attack$ and $\support$ can be organised as incidence
matrices $G_a$ and $G_s$. A combined attack/support graph then
leads to a \wasa in our sense by setting $G:=G_s-G_a$.

\subsection{Recursive Sigmoid Aggregation Semantics}
\label{sec:rsig}

Both the h-categorizer semantics of \cite{amgoudevaluation} (for attack
relations) and the aggregation based semantics of
 \cite{DBLP:conf/ijcai/AmgoudB16} (for support relations) work with
a summation of the attacks and supports respectively. They can be combined
into one semantics for \wasa\!\!\!s as follows (note that weights
are restricted to $[0,1]$):
\begin{definition}[Recursive Sigmoid Aggregation Function]
Let \defaultarggraph be a \wasa such that $w:\mathcal{A}\to[0,1]$. 
For $i\in\mathbb{N}$, the recursive sigmoid aggregation function $\fsemrsig_i$ from \argset to $[0, 1]$ is defined as follows: for any 
$a\in \argset$, for $i \in \{0, 1, 2, \ldots  \} $, if $i = 0$, then $\fsemrsig_i(a) = w(a)$, otherwise

\[
 \fsemrsig_i(a) =
\begin{cases}
\frac{w(a)}{1-s^i(a)} & \textit{if } s^i(a) \leq 0\\
\frac{w(a)+ s^i(a)}{1 + s^i(a)}&      \textit{if } s^i(a) \geq 0\\
\end{cases}
\]

where, for any $a\in \argset$:  
	
	\[ s^i(a) = \sum_{b\in \supporters{}{a}} \fsemrsig_{i-1}(b) - 
				\sum_{b\in \attackers{}{a}} \fsemrsig_{i-1}(b)	
\]
\end{definition}

\begin{definition}[Recursive Sigmoid Aggregation Semantics]
	The \emph{recursive sigmoid aggregation semantics} is a function \semrsig transforming any \wasa \defaultarggraph 
	 into a weighting on \argset such that for any $a\in \argset $
	\[ \accdegr{\semrsig}{}{a} = \lim_{i\rightarrow \infty}\fsemrsig_i(a)
	\]

\end{definition}

\begin{example}
	The function $\fsemrsig_i$ does not converge in general. 
Consider the \wasa
{
\xymatrixcolsep{2pt}
$$\xymatrix{
\frac{3}{4}&a\ar@{<->}[rrrr]\ar@{{*}-{*}}[d] &&&&
b \ar@{{*}-{*}}[d]&\frac{1}{4}\\
\frac{3}{4}&c\ar@{<->}[rrrr] &&&& d&\frac{1}{4}
}$$
}
Then $\fsemrsig_{2i}=\matr{\frac{3}{4}&\frac{1}{4}&\frac{3}{4}&\frac{1}{4}}^T$
and $\fsemrsig_{2i+1}=\matr{\frac{1}{2}&\frac{1}{2}&\frac{1}{2}&\frac{1}{2}}^T$.
\end{example}
This shows that a na\"ive combination of the 
semantics of attacks from \cite{amgoudevaluation}
with the semantics of supports from \cite{DBLP:conf/ijcai/AmgoudB16}
is not possible.

In the next two subsections, we will study two modifications of this
semantics that hopefully will lead to convergence.

However, as shown in \cite{DBLP:conf/ijcai/AmgoudB16}, $\fsemrsig_i$
\emph{does} converge when graphs are restricted to support relations
only.  This is called aggregation-based semantics in
\cite{DBLP:conf/ijcai/AmgoudB16}.
\begin{theorem}\label{thm:French-aggr-properties}
Aggregation-based semantics \cite{DBLP:conf/ijcai/AmgoudB16}, 
which is $\accdegrvec{\semrsig}{}$ 
restricted to graphs with only support relations,
satisfies all of our mandatory characteristics.
\end{theorem}

In \cite{DBLP:conf/ijcai/AmgoudB16}, two more semantics are defined.
One is top-based semantics. It differs from aggregation-based
semantics in that multiple supports for a given argument are not
summed up, but only the maximum support is considered. Hence, the
number of supported is ignored, only their quality matters. By
contrast, reward-based semantics ``favours the number of supporters
over their quality'' \cite{DBLP:conf/ijcai/AmgoudB16}. This is
achieved by defining $f_i$ as, in binary representation, $0.111\ldots 1$
($m$ ones, where $m$ is the number of supporters), plus the strength
of the supporters which is normalised in a way such that it has effect
only in the subsequents bits.  Both semantics fulfil specialised
optional axioms. The principles of these semantics can easily be
carried over to our (sigmoid) direct aggregation semantics. However,
these semantics are not so interesting for our use case. Moreover,
top-based semantics does not even satisfy all the axioms which are
said to be mandatory in \cite{DBLP:conf/ijcai/AmgoudB16}. Therefore,
we refrain from developing these semantics in detail here.

\subsection{Recursive Damped Aggregation Semantics}
Our first attempt to modify the sigmoid aggregation semantics in order to make it convergent removes the sigmoid character of the functions and uses linear functions instead. Convergence is ensured by dividing
 $s^i$ by a damping factor $d>\ind(G)$, such that its components range between -1 and 1.
This allows us to get rid of $s^i$ in the denominator and use 
functions linear in $s^i$ instead of a sigmoid function.
\begin{definition}[Recursive damped aggregation function]
Let \defaultarggraph be a \wasa such that $w:\mathcal{A}\to[0,1]$. Let the damping factor $d$ be such that $d \geq 1$. 
For $i\in\mathbb{N}$, the recursive damped aggregation function $\fsemrd_i$ from \argset to $[0, 1]$ is defined as follows: for any 
$a\in \argset$, for $i \in \{0, 1, 2, \ldots  \} $, if $i = 0$, then $\fsemrd_i(a) = w(a)$, otherwise

\[
 \fsemrd_i(a) =
\begin{cases}
w(a)(1+s^i(a)) & \textit{if } s^i(a) \leq 0\\
w(a)+(1-w(a))s^i(a)   & \textit{if } s^i(a) \geq 0\\
\end{cases}
\]

where, for any $a\in \argset$:  
	
	\[ s^i(a) =\frac{1}{d}\times\left (\sum_{b\in \supporters{}{a}} \fsemrd{i-1}(b) - 
				\sum_{b\in \attackers{}{a}} \fsemrd{i-1}(b)\right )
\]
An short matrix notation is

$$\begin{array}{l}
\fsemrd_i = w+Diag(p(w,s^i))s^i\\
\text{where }s^i = \frac{1}{d}G\fsemrd_{i-1}\\
\qquad~~ p(w,s)(x)=\begin{cases}
                       w(x) & \textit{if }s(x)<0\\
                       1-w(x) & \textit{if }s(x)\geq 0
 		    \end{cases}
\end{array}$$
Here, $Diag$ uses a vector to fill the diagonal of a matrix, which is
otherwise zero.
\end{definition}

\begin{conjecture}
For $d>\ind(G)$, $\lim_{i\to\infty}\fsemrd_i$ converges.
\end{conjecture}

\subsection{Damped Dogged Semantics}
\label{sec:dogged}

Our second modification of the sigmoid aggregation semantics that shall
reach convergence keeps the sigmoid character but avoids the case distinction
of the previous subsections.
\begin{definition}[Dogged Function]
Let $\sigma$ be any sigmoid function, see Sect.~\ref{sec:sigmoid-direct}
for examples.
Let \defaultarggraph be a \wasa such that $w:\mathcal{A}\to[0,1]$. 
Let the damping factor $d$ be such that $d \geq 1$.  
	The dogged function $f^{\sigma}_i$ from \argset to $[0, 1]$ is defined as follows: for any 
	$a\in \argset$, for $i \in \{0, 1, 2, \ldots  \} $, if $i = 0$, then $f^{\sigma}_i(a) = w(a)$, otherwise

	\[
	 f_i^{\sigma}(a) =
	\begin{cases}
	1 & \textit{if } w(a) = 1 \\
	\sigma(s^i(a)+ \sigma^{-1}(w(a)))&      \textit{if } 0  <  w(a)< 1\\
	0 & \textit{if } w(a) = 0 \\ 
	\end{cases}
	\]

	where, for any $a\in \argset$:  
	
		\[ s^i(a) = \frac{1}{d}\times\left (\sum_{b\in \supporters{}{a}} f^{\sigma}_{i-1}(b) - 
					\sum_{b\in \attackers{}{a}} f^{\sigma}_{i-1}(b)	\right )
	\]

	\end{definition}

\begin{example}
For $d=1$ and $\sigma_1(x)=\frac{1}{1+e^{-x}}$ or
$\sigma_2(x)=\frac{\arctan(x)}{\pi}+\frac{1}{2}$
or $\sigma_3(x)=\frac{1+|x|+x}{2(1+|x|)}$, 
$\lim_{i\to\infty}f_i^{\sigma_j}$ does not converge in general.
Consider the following graph
$$\xymatrix{
a \ar@{{*}-{*}}[rrr] \ar@{{*}-{*}}[dd] \ar@{{*}-{*}}[dr] \ar@{{*}-{*}}[drr] \ar@{-{*}}@(l,dl)[ddrrr] &&&
b \ar@{-{*}}[dll] \ar@{{*}-{*}}[dl] \ar@{-{*}}[dd] \ar@{{*}-{*}}@(r,dr)[ddlll] \\
& d \ar@{{*}-{*}}[r] \ar@{{*}-{*}}[dl] \ar@{{*}-{*}}[drr] &
e \ar@{{*}-{*}}[dll] \ar@{{*}-{*}}[dr] &\\
f \ar@{{*}-}[rrr] &&& c
}$$
with initial plausibility $0.85$ for every node. For large enough $i$, we have
$$
\begin{array}{l}
f_{2i}^{\sigma_1}\approx\matr{0.386435 & 0.529751 & 0.357394 &0.236454 & 0.236454 & 0.236454}^T\\
f_{2i+1}^{\sigma_1}\approx\matr{0.621398 & 0.705838 & 0.585527 & 0.497027 & 0.497027 & 0.4970277}^T
\end{array}$$
$\sigma_2$ and $\sigma_3$ exhibit a similar behaviour (with the same graph
and initial plausibilities).

\end{example}

\begin{conjecture}
For $d>\ind(G)$, $\lim_{i\to\infty}f_i^{\sigma}$ converges.
\end{conjecture}

\section{Conclusion and Future work} \label{sec:future}

We have shown that bipolar argumentation graphs can be equipped with a
weighting (ranked-based) semantics, both for weights ranging over real
numbers as well as for weights in the range $(0,1)$. The neutral value
is $0$ in the former case and $\frac{1}{2}$ in the latter case.  Both
semantics fulfil suitable characteristics. E.g.\ the computed
acceptability degree of an argument is monotonic both in the initial
plausibility and in the set of supporting arguments. These
characteristics have been taken from the literature and suitably
generalised and strengthened.

The comparison of our semantics to related work in the literature
(see also Table~\ref{fig:semantics-overview})
naturally lead to further (recursively defined) bipolar semantics (see
section~\ref{sec:comparison}), the convergence of which is still
open. Note that these semantics are defined over $[0,1]$ but still use
$0$ as the neutral value. We think that it is conceptually more
convincing to use $\frac{1}{2}$ as neutral value for bipolar semantics
with weights in $[0,1]$, as we did for Sigmoid direct aggregation
semantics in section~\ref{sec:sigmoid-direct}.

\begin{figure}
\begin{tabular}{|p{3cm}|c|c|c|c|c|}\hline
Semantics                     & weight range & neutral value & convergent & bounded &reverse impact  \\\hline\hline
Direct aggregation            & $\mathbb{R}$ & 0             & yes & no & yes \\\hline
Sigmoid direct aggregation    & $(0,1)$      & $\frac{1}{2}$ & yes & no & yes\\\hline
Recursive sigmoid aggregation & $[0,1]$      & 0             & no  & yes & no \\\hline
Recursive damped aggregation  & $[0,1]$      & 0             & ?   & yes & no\\\hline
Damped dogged                 & $[0,1]$      & 0             & ?   & yes & no\\\hline
\end{tabular}
\caption{\label{fig:semantics-overview}Overview of the different semantics.}
\end{figure}

Future work should consider the questions whether a bipolar semantics
for compact intervals like $[0,1]$ or
$\mathbb{R}\cup\{-\infty,\infty\}$ are possible. One way would be to
prove the bipolar recursive semantics developed in in
section~\ref{sec:comparison} to be convergent.  Another way would be
to extend sigmoid direct aggregation semantics (see
section~\ref{sec:sigmoid-direct}) from $(0,1)$ to $[0,1]$. An obvious
solution would define $\sigma(-\infty)=0$ and
$\sigma(\infty)=1$. However, then a major difficulty is the
development of a suitable arithmetics for the extended real line that
keeps Thm.~\ref{thm:direct-aggr-properties} true. 

Another future direction is to equip attack and support relations with
weights, e.g.\ in the interval $[-1,1]$. See
\cite{DBLP:journals/ai/DunneHMPW11,DBLP:conf/kr/Coste-MarquisKMO12}
for work in this direction, but in a different context: only attacks
are equipped with weights, not the arguments themselves.

Also the study of characteristics leaves some open questions. For
example, is it possible to generalise Counting in a way that one does
not consider exactly the same set of attackers, but a set of
comparable attackers?

Also, we would like to use our framework to define a semantics for
the Argument Interchange Format (AIF, \cite{RahwanReed09}) that is
simpler and more direct than the one given in the literature
\cite{DBLP:journals/logcom/BexMPR13}.

Finally, large argumentation graphs will benefit from a modular
design; e.g.\ in \cite{BetzCacean2012} they are often divided into
subgraphs, e.g.\ by drawing boxes around some groups of arguments.
The characteristics of our semantics suggest that modularity
can be obtained by substituting suitable subgraphs with discrete
graphs whose arguments are initially weighted with their degrees
in the original graph.

\bibliographystyle{unsrt}
\bibliography{../bibliography/argTheo}

\appendix
\section*{Omitted proofs}
\renewcommand{\accdegr}[3]%
	{\ensuremath{\texttt{Deg}
	^{\ifthenelse{\equal{#1}{}}{\bs}{#1}}%
	_{\ifthenelse{\equal{#2}{}}{\arggraph}{#2}}
	{\ifthenelse{\equal{#3}{}}{ }{(#3)}}
	\xspace}}

\begin{proof}{ of Thm. \ref{thm:Causality}}
Let \defaultarggraph be a \wasa such that $a$ in \argset. Further assume that there is no $b$ in
\argset such that  	$\accdegr{}{}{b} \neq \neutrald$ and $b \in \attackers{}{a} \cup \supporters{}{a}$.  
Let $c_1, \ldots, c_n$ be an enumeration of all arguments in $\attackers{}{a} \cup 
 \supporters{}{a}$. Neutrality entails $\accdegr{}{\arggraph|_{c_1, \ldots c_n}}{} = 
 \accdegr{}{}{}$. Since $\attackers{\arggraph|_{c_1, \ldots, c_n}}{a} = 
 \supporters{\arggraph|_{c_1, \ldots ,c_n}}{a} = \emptyset$, Conservativity entails 
$\accdegr{}{\arggraph|_{c_1, \ldots , c_n}}{a} = w(a)$.
Thus, $\accdegr{}{}{a} = w(a)$. 
\end{proof}

\begin{proof}{ of Thm.~\ref{thm:Dummy}}
Assume \bwsa \defaultarggraph , such that
		 $w(a) = w(b)$,
		 $\supporters{}{a} = \supporters{}{b} \cup \{ x \} $ and 
		 $\accdegr{}{}{x} = \neutrald$.
	Note that
	 $\attackers{\arggraph |_{x}}{a} = \attackers{\arggraph |_{x}}{b}$,
	 $\supporters{\arggraph |_{x}}{a} = \supporters{\arggraph |_{x}}{b} $.
		 By definition 		 $\attackers{}{a}= \attackers{\arggraph |_{x}}{a} = \attackers{}{b} =\attackers{\arggraph |_{x}}{b} = \emptyset$.	 
	 Thus, Parent Monotony entails 
	 $\accdegr{}{\arggraph |_{x}}{a} = \accdegr{}{\arggraph |_{x}}{b} $. 
	  Because of Neutrality $\accdegr{}{}{a} = \accdegr{}{\arggraph |_{x}}{a}$
	  and
	   $\accdegr{}{\arggraph |_{x}}{b} = \accdegr{}{}{b}$.
	   Therefore,   $\accdegr{}{}{a} = 	 \accdegr{}{}{b}$.
\end{proof}

\begin{proof}{ of Thm.~\ref{thm:stickiness}}
	Assume a \bwsa \defaultarggraph such that 
	 $w(a) = w(b)$,
	 $\supporters{}{a} \backslash \supporters{}{b} = \{ x \} $,
	 $\supporters{}{b} \backslash \supporters{}{a} = \{ y \} $,
	 $\accdegr{}{}{x} > \accdegr{}{}{y}$,  and
	 $\accdegr{}{}{b} = 1$.
	Let $\supporters{}{a} 
	\backslash   \{ x \}  = \supporters{}{b} \backslash  \{ y \}  = \{ c_1, \ldots, c_n \}$. 
	Let $\argset ^\dagger = \{d_0, \ldots, d_n, e_0, \ldots, e_n  \} $ be a set of arguments that 
	 is disjoint with $\argset$.
	 $\arggraph ^\dagger$ is defined as following: 
	   \[
		\argset^{\dagger} = 
		\left(\begin{smallmatrix}
		\argset \\
		d_0 \\
		d_1 \\
		\vdots \\
		d_n \\
		e_0 \\
		e_1 \\
		\vdots \\
		e_n
		\end{smallmatrix}\right),
	    G 	^{\dagger}=
	  	\left(  \begin{smallmatrix}
			G & 0 \\
			0 & 0
		    \end{smallmatrix}\right) , 	
	 	w^{\dagger} = 
	 	\left(\begin{smallmatrix}
	 	w \\
		 \accdegr{}{}{x}\\
	 	 \accdegr{}{}{c_1}\\
		 \vdots \\
		  \accdegr{}{}{c_n} \\
		 \accdegr{}{}{y}\\
 	 	 \accdegr{}{}{c_1}\\
 		 \vdots \\
 		  \accdegr{}{}{c_n}		  
	 	\end{smallmatrix}\right)
	   \] 
It follows from Independence that for any $z\in \argset$, 
$ \accdegr{}{\arggraph ^{\dagger}}{z} = \accdegr{}{}{z}$. Now we apply Interchangeability $2n+2$ times in order to disconnect $a$ and $b$ from $x, c_1, \dots, c_n$ and $y, c_1, \dots, c_n$, respectively. In the resulting \bwsa $\arggraph ^\ddagger$ it holds that
$\supporters{\arggraph ^\ddagger}{a} = 	\{ d_0, d_1, \dots, d_n \} $ and
 $\supporters{\arggraph ^\ddagger}{b} = \{e_0, e_1, \ldots, e_n\}$, while 
 $\accdegr{}{\arggraph ^\ddagger}{a} = \accdegr{}{\arggraph ^\dagger}{a}$  and 
 $\accdegr{}{\arggraph ^\ddagger}{b} = \accdegr{}{\arggraph ^\dagger}{b}$.
  
 Now consider the  \bwsa $\arggraph ^ \prime$ that is the result of removing all outgoing arcs from $a$ and $b$ in  $\arggraph ^\ddagger$. 
 $\arggraph ^\prime$ consists of three disconnected subgraphs: (i) $a$ and its supporters $d_0, d_1, \ldots, d_n$; (ii) $b$ and its supporters $e_0, e_1, \ldots, e_n$; and (iii) the remainder of $\arggraph$ without $a, b$.
 Since there is no argument $ z \in \argset$ such that  $z\in \backers{\arggraph ^\ddagger}{a}$ or $z\in \detractors{\arggraph ^\ddagger}{a}$, it follows from Directionality that $\accdegr{}{\arggraph ^\prime}{a} = \accdegr{}{\arggraph ^\ddagger}{a}$.
For the same reason $\accdegr{}{\arggraph ^\prime}{b} = \accdegr{}{\arggraph ^\ddagger}{b}$.

	 $\arggraph ^a = \langle \argset ^{a}, G ^{a} , w^a\rangle$  and  $\arggraph ^b = \langle 
  \argset ^{b}, G ^{b} , w^b\rangle$ are defined to match two of the subgraphs  in $\arggraph ^\prime$. Note that by 
 assumption $w(a) = w(b)$. Further, $G 	^{a}= G 	^{b}$, since both graphs consists just of one 
	 node that is supported by $n+1$ other nodes.
	   \[
		\argset^{a} = 
		\left(\begin{smallmatrix}
		a \\
		d_0 \\
		d_1 \\
		\vdots \\
		d_n \\
		\end{smallmatrix}\right),
	    G 	^{a}=
	  	\left(  \begin{smallmatrix}
			0 & 1 & \cdots & 1  \\
			0 & 0 & \cdots & 0   \\
			\vdots & & \ddots & \vdots \\ 
			0 & 0 & \cdots & 0   		
		    \end{smallmatrix}\right) , 	
	 	w^{a} = 
	 	\left(\begin{smallmatrix}
	 	w(a) \\
		 \accdegr{}{}{x}\\
	 	 \accdegr{}{}{c_1}\\
		 \vdots \\
		  \accdegr{}{}{c_n} \\
	 	\end{smallmatrix}\right)
	   \] 

	   \[
		\argset^{b} = 
		\left(\begin{smallmatrix}
		b \\
		e_0 \\
		e_1 \\
		\vdots \\
		e_n \\
		\end{smallmatrix}\right),
	    G 	^{b}= G 	^{a} , 	
	 	w^{b} = 
	 	\left(\begin{smallmatrix}
	 	w(a) \\
		 \accdegr{}{}{y}\\
	 	 \accdegr{}{}{c_1}\\
		 \vdots \\
		  \accdegr{}{}{c_n} \\
	 	\end{smallmatrix}\right)
	   \] 
 Because of Independence and Anonymity it follows 
	 $ \accdegr{}{\arggraph ^a}{a} = \accdegr{}{\arggraph ^\prime}{a} $	 
	 and
	 $ \accdegr{}{\arggraph ^b}{b} = \accdegr{}{\arggraph ^\prime}{b} $.
	 In a last step we rename the arguments in $\arggraph ^b$ and define 
%
	$\arggraph ^{r}$ such that 
	$\argset^{r}  = \argset^{a},  
	  G^{r} = G^{a},
	   w^{r} = w^{b}$.

Note that $w^a(a) = w^r(a), \attackers{\arggraph ^{a}}{a} =   \attackers{\arggraph ^{r}}{a} = \emptyset$ and $\supporters{\arggraph ^{a}}{a} =   \supporters{\arggraph ^{r}}{a}$. Further, 
since by assumption 	 $\accdegr{}{}{y} < \accdegr{}{}{x}$, it follows, for any argument $z$ 
 in $\supporters{\arggraph ^{r} }{a} $, that $w^r(z)\leq w^a(z)$, and, thus, 
$  \accdegr{}{\arggraph ^{r}}{z} \leq   \accdegr{}{\arggraph ^{a}}{z}$ (because of Conservativity). Hence, Parent Monotony entails 
$ \accdegr{}{\arggraph ^{r}}{a} \leq   \accdegr{}{\arggraph ^{a}}{a}$. 
 Anonymity gives us that
 	 $ \accdegr{}{\arggraph ^{r} }{} = \accdegr{}{\arggraph ^b}{} $. Thus,
 $ \accdegr{}{\arggraph ^{r}}{a} =
 	 \accdegr{}{\arggraph ^b}{b} =
 	 \accdegr{}{\arggraph ^\prime}{b} =
 	 \accdegr{}{\arggraph ^\ddagger}{b}	=
 	 \accdegr{}{\arggraph ^\dagger}{b}	=
 	 \accdegr{}{}{b} = 1
	$.
 	 Further,
	 $ 	 \accdegr{}{\arggraph ^a}{a} =
	 	 \accdegr{}{\arggraph ^\prime}{a} =
	 	 \accdegr{}{\arggraph ^\ddagger}{a}	=
	 	 \accdegr{}{\arggraph ^\dagger}{a}	=
	 	 \accdegr{}{}{a}
	 $.
 Therefore, $1 \leq 	 	 \accdegr{}{}{a}$.
  Since $\accdegr{}{}{a}\in [0,1]$, $\accdegr{}{}{a} = 1$.
\end{proof}

\begin{proof}{ of Thm.~\ref{thm:French-characteristics}}
\noindent Anonymity, Independence, Equivalence: clear, since our
  definition (when restricted to graphs with supports relations only)
  is the same definition as in \cite{DBLP:conf/ijcai/AmgoudB16}.

\medskip\noindent Non-Dilution: is easily seen to be a special case of
Directionality.

\medskip\noindent Monotony and Counting and Reinforcement: this
follows respectively from Parent Monotony and Impact and Reinforcement, using
Independence, Interchangeability, Anonymity and Conservativity, as in
the proof of Thm. \ref{thm:stickiness} in order to switch from a
single-graph property to a two-graph property.

\medskip\noindent Dummy: see Thm.~\ref{thm:Dummy}.

\medskip\noindent Minimality: same as our Conservativity.

\medskip\noindent Strengthening: given arguments $a$, $b$ with
$w(a)<1$ and $b \in \supporters{}{a}$, we need to show that
$\accdegr{}{}{a}>w(a)$.  Now let
$\supporters{}{a}=\{b_1,\ldots,b_n\}$. Without loss of generality, we
can assume that $a$ is not among the $\{b_1,\ldots,b_n\}$. If not, use
Independence and Interchangeability as in the proof of
Thm. \ref{thm:stickiness} in order to replace $a$ in
$\supporters{}{a}$ by an equivalently weighted node.  Now by
Conservativity,
$\accdegr{}{\arggraph|_{b_1,\ldots,b_n}}{a}=w(a)$. Repeated
application of Impact (noting that there are no attacks, hence no
detractors) gives $\accdegr{}{}{a}>w(a)$, unless $w(a)=1$ (which
however we have excluded).

\medskip\noindent Strengthening Soundness: this follows easily from
Causality.

\medskip\noindent Coherence: in case of $\accdegr{}{}{b}>0$, this
follows easily from Initial Monotony. If $\accdegr{}{}{b}=0$,
use Conservativity and Parent Monotony to show that $\accdegr{}{}{a}>0$.

\medskip\noindent Boundedness: we have called this stickiness (in order
to distinguish it from our boundedness), and stickiness is proven in
Thm.~\ref{thm:stickiness}.
\end{proof}

\renewcommand{\accdegr}[3]{\ensuremath{\texttt{Deg}^{\ifthenelse{\equal{#1}{}}{ad}{#1}}%
_{\ifthenelse{\equal{#2}{}}{\arggraph}{#2} }}(#3)\xspace}

\begin{proof}{ of Thm.~\ref{thm:direct-aggr-properties}}
\noindent Anonymity: The isomorphism $f$ can be organised into a permutation
  matrix $P$ such that $w'=Pw$ and $G'=PGP^{-1}$.  These equations
  transform $\accdegrvec{\semdir}{\arggraph,d} = w+\frac{1}{d}G\accdegrvec{\semdir}{\arggraph,d}$ (holding by
  Thm.~\ref{thm:direct-aggregation-equation}) into
  $P\accdegrvec{\semdir}{\arggraph,d}=w'+\frac{1}{d}G'P\accdegrvec{\semdir}{\arggraph,d}$. Hence, again by by
  Thm.~\ref{thm:direct-aggregation-equation},
  $P\accdegrvec{\semdir}{\arggraph,d}=\accdegrvec{\semdir}{\arggraph^\prime,d}$.

\medskip\noindent Equivalence:
Since the bijective functions $f$ and $g$ must have disjoint
domains and images, they can be combined into a
permutation matrix $P$ that behaves as identity outside the union of 
$f$' and $g$'s domains. The assumptions then
can be written more compactly as 
\begin{eqnarray}
w(a)=w(b)\label{ass:w}\\
P\accdegrvec{\semdir}{\arggraph,d}=\accdegrvec{\semdir}{\arggraph,d}\label{ass:P}\\
\parent{}{a}P=\parent{}{b}\label{ass:infl}
\end{eqnarray}
Now by Thm.~\ref{thm:direct-aggregation-equation},
$\accdegr{\semdir}{\arggraph,d}{b}=w(b)+\frac{1}{d}G\accdegrvec{\semdir}{\arggraph,d}(b)$. This is
$w(b)+\frac{1}{d}\parent{}{b}\accdegrvec{\semdir}{\arggraph,d}$, which by (\ref{ass:w}) and
(\ref{ass:infl}) is $w(a)+\frac{1}{d}\parent{}{a}P\accdegrvec{\semdir}{\arggraph,d}$.  By
(\ref{ass:P}), we arrive at $w(a)+\frac{1}{d}\parent{}{a}\accdegrvec{\semdir}{\arggraph,d}$, which
is $w(a)+\frac{1}{d}G\accdegrvec{\semdir}{\arggraph,d}(a)$, and again by 
Thm.~\ref{thm:direct-aggregation-equation}, this is $\accdegr{\semdir}{\arggraph,d}{a}$.

\medskip\noindent Directionality: We prove the following lemma
$$\forall y\in\backers{}{x}\cup\detractors{}{x}\cup\{x\},\ G^iw(y)=G'^iw(y)$$
by induction over $i$. For $i=0$, both sides of the equation are $w(y)$.
Now let us prove the statement for $i+1$: $G^{i+1}w(y)=GG^iw(y)=
\parent{}{y}G^iw$. Since all for $y'\in\parent{}{y}$, we have
$y'\in\backers{}{x}\cup\detractors{}{x}\cup\{x\}$, by induction
hypothesis, $G^iw(y')=G'^iw(y')$ for such $y'$, and hence
$\parent{}{y}G^iw=\parent{}{y}G'^iw$.  Now by the assumption, $y\not=
a_j$, and hence $\parent{}{y}=\parent{\arggraph^\prime}{y}$.  Thus
$\parent{}{y}G'^iw=\parent{\arggraph^\prime}{y}G'^iw=G'G'^iw(y)=G'^{i+1}w(y)$.
Hence altogether, $G^{i+1}w(y)=G'^{i+1}w(y)$, and the lemma is proved.
Since the lemma applies in particular to $y=x$, we get
$\accdegr{\semdir}{\arggraph,d}{x} =\sum_{i=0}^\infty \ (\frac{1}{d}G)^iw(x)=\sum_{i=0}^\infty \ (\frac{1}{d}G')^iw(x)
=\accdegr{\semdir}{\arggraph^\prime,d}{x}$.

\medskip\noindent Conservativity: By Thm.~\ref{thm:direct-aggregation-equation},
$\accdegr{\semdir}{\arggraph,d}{a}=w(a)+\frac{1}{d}\parent{}{a}\accdegrvec{\semdir}{\arggraph,d}$.  By the
assumption, $\parent{}{a}=\matr{0\ldots 0}$, so
$\accdegr{\semdir}{\arggraph,d}{a}=w(a)$.

\medskip\noindent Initial Monotony:
Using the assumptions and Thm.~\ref{thm:direct-aggregation-equation},
we have
$\accdegr{\semdir}{\arggraph,d}{a}=w(a)+\frac{1}{d}\parent{}{a}\accdegrvec{\semdir}{\arggraph,d}
>w(b)+\frac{1}{d}\parent{}{a}\accdegrvec{\semdir}{\arggraph,d}=w(b)+\frac{1}{d}\parent{}{b}\accdegrvec{\semdir}{\arggraph,d}=\accdegr{\semdir}{\arggraph,d}{b}$.

\medskip\noindent Neutrality: By Thm.~\ref{thm:direct-aggregation-equation},
$\accdegrvec{\semdir}{\arggraph,d} = w+\frac{1}{d}G\accdegrvec{\semdir}{\arggraph,d}$.
Since $\accdegr{\semdir}{}{a}= 0$, this is $w+\frac{1}{d}\overline{G}\accdegrvec{\semdir}{\arggraph,d}$,
where $\overline{G}$ is $G$ with the column for $a$ replaced by zeros.
Replacing (in $\overline{G}$) also the row for $a$ with zeros leads to $G'$,
where $\arggraph |_{a} =\langle {\argset ^\prime, G^\prime, w^\prime}\rangle $.
Therefore, $w+\frac{1}{d}\overline{G}\accdegrvec{\semdir}{\arggraph,d}$
and $w'+\frac{1}{d}G'\accdegrvec{\semdir}{\arggraph,d}$ are equal except
possibly for $a$. However, on $a$ they both are $0$.
Thus $\accdegrvec{}{}{} = \accdegrvec{}{\arggraph |_{a}}{}$.

\medskip\noindent Parent Monotony:
We have 
$$\begin{array}{cl}
  & \accdegr{\semdir}{\arggraph,d}{a}\\
\stackrel{\text{Thm.~\ref{thm:direct-aggregation-equation}}}= & w(a)+\frac{1}{d}\parent{}{a}\accdegrvec{\semdir}{\arggraph,d}\\
\stackrel{\ref{monotony:w}}{=} & w'(a)+\frac{1}{d}\parent{}{a}\accdegrvec{\semdir}{\arggraph,d}\\
= & w'(a) + \frac{1}{d} \left 
       (\sum_{b\in \supporters{}{a}} \accdegr{\semdir}{\arggraph,d}{b} - 
	\sum_{c\in \attackers{}{a}} \accdegr{\semdir}{\arggraph,d}{c}\right)\\
\stackrel{\ref{monotony:subset}}{\leq} &
w'(a) + \frac{1}{d} \left 
       (\sum_{b\in \supporters{}{a}} \accdegr{\semdir}{\arggraph,d}{b} - 
	\sum_{c\in \attackers{\arggraph'}{a}} \accdegr{\semdir}{\arggraph,d}{c}\right)\\
\stackrel{\ref{monotony:att},\ref{monotony:deg}}{\leq} &
w'(a) + \frac{1}{d} \left 
       (\sum_{b\in \supporters{}{a}} \accdegr{\semdir}{\arggraph',d}{b} - 
	\sum_{c\in \attackers{\arggraph'}{a}} \accdegr{\semdir}{\arggraph',d}{c}\right)\\
\stackrel{\ref{monotony:subset}}{\leq} &
w'(a) + \frac{1}{d} \left 
       (\sum_{b\in \supporters{\arggraph'}{a}} \accdegr{\semdir}{\arggraph',d}{b} - 
	\sum_{c\in \attackers{\arggraph'}{a}} \accdegr{\semdir}{\arggraph',d}{c}\right)\\
= & w(a)+\frac{1}{d}\parent{\arggraph'}{a}\accdegrvec{\semdir}{\arggraph',d}\\
\stackrel{\text{Thm.~\ref{thm:direct-aggregation-equation}}}=  & \accdegr{\semdir}{\arggraph',d}{a}
\end{array}$$

\medskip\noindent Impact: We only show the first half, the second half
being dual.
Let $\arggraph|_{b}=\langle {\argset, G', w'}\rangle$.
We show the following lemma:
$$\begin{array}{l}
x\not= b\wedge b\not\in\backers{}{x} \Rightarrow \accdegr{\semdir}{\arggraph,d}{x}\leq\accdegr{\semdir}{\arggraph|_{b},d}{x}\\
b\not\in\detractors{}{x} \Rightarrow \accdegr{\semdir}{\arggraph,d}{x}\geq\accdegr{\semdir}{\arggraph|_{b},d}{x}
\end{array}$$
This follows from
\begin{eqnarray}
x\not= b\wedge b\not\in\backers{}{x} \Rightarrow G^iw(x)\leq G'^iw'(x)
\label{back-leq}\\
b\not\in\detractors{}{x} \Rightarrow G^iw(x)\geq G'^iw'(x)
\label{detr-geq}
\end{eqnarray}
which we show simultaneously by induction over $i$. The induction base 
$i=0$ follows since $w$ and $w'$ only differ in that $w'(b)=0$.
Concerning (\ref{back-leq}) for $i+1$, 
$G^{i+1}w(x)=GG^iw(x)=\parent{}{x}G^iw\stackrel{(*)}{\leq}\parent{}{x}G'^iw'\stackrel{(**)}{=}\parent{\arggraph|_{b}}{x}G'^iw'= G'^iw'(x)=G'G'^iw'(x)=G'^{i+1}w'(x)$.
Now for $(*)$, we need to show 
\begin{itemize}
\item $y\in\supporters{}{x}$ implies $G^iw(y)\leq G'^iw(y)$. This
  follows from the induction hypothesis since $y\not= b\wedge
  b\not\in\backers{}{y}$.  The latter follows since $y=b$ as well as
  $b\in\backers{}{y}$ both would imply $b\in\backers{}{x}$.
\item $y\in\attackers{}{x}$ implies $G^iw(y)\geq G'^iw(y)$. This
  follows from the induction hypothesis since
  $b\not\in\detractors{}{y}$.  The latter follows since
  $b\in\detractors{}{y}$ would imply $b\in\backers{}{x}$.
\end{itemize}
$(**)$ holds because $\parent{}{x}$ and $\parent{\arggraph|_{b}}{x}$
can only differ in the column for $b$ --- but $G'^iw$ has a zero
row for $b$.

\medskip\noindent (\ref{detr-geq}) is shown similarly.

\medskip\noindent Finally,
$\accdegr{\semdir}{\arggraph,d}{a}=\parent{}{a}\accdegrvec{\semdir}{\arggraph,d}\stackrel{(***)}<
\parent{\arggraph|_{b}}{a}\accdegrvec{\semdir}{\arggraph|_{b},d}=\accdegr{\semdir}{\arggraph|_{b},d}{a}$. $(***)$
holds because
\begin{itemize}
\item for $y\in\supporters{}{a}$, we have $y\not= b$ (because $b$
  attacks $a$) and $b\not\in\backers{}{y}$. By the lemma,
  $\accdegr{\semdir}{\arggraph,d}{y}\leq\accdegr{\semdir}{\arggraph|_{b},d}{y}$.
\item for $y\in\attackers{}{a}$, we have $b\not\in\detractors{}{y}$. By
  the lemma, $\accdegr{\semdir}{\arggraph,d}{y}\leq\accdegr{\semdir}{\arggraph|_{b},d}{y}$.
\end{itemize}
Since $\parent{\arggraph|_{b}}{a}$ is $\parent{}{a}$ with the
column for $b$ set to $0$, but $\parent{}{a}(b)<0$ and
$\accdegr{\semdir}{\arggraph,d}{b}>0$, $(***)$ follows.

\medskip\noindent Reinforcement: Concerning 1, the assumptions imply
$\parent{\arggraph}{a}\accdegrvec{\semdir}{\arggraph,d}>\parent{\arggraph'}{a}\accdegrvec{\semdir}{\arggraph,d}$.
 Then $\accdegr{\semdir}{\arggraph,d}{a}=\parent{\arggraph}{a}\accdegrvec{\semdir}{\arggraph,d}>\parent{}{a}\accdegrvec{\semdir}{\arggraph',d}\stackrel{(*)}=\parent{\arggraph'}{a}\accdegrvec{\semdir}{\arggraph',d}=\accdegr{\semdir}{\arggraph',d}{a}$, where $(*)$ holds because the attackers
and supports of $a$ agree for $\arggraph$ and $\arggraph'$.
The proof of 2 is analogous.

\medskip\noindent Causality: follows from Thm.~\ref{thm:Causality}.

\medskip\noindent Neutralisation: Using Thm.~\ref{thm:direct-aggregation-equation}, we
get
$\accdegrvec{\semdir}{\arggraph,d}=w+\frac{1}{d}G\accdegrvec{\semdir}{\arggraph,d}\stackrel{(*)}=w+\frac{1}{d}G'\accdegrvec{\semdir}{\arggraph,d}$,
where $(*)$ follows from the assumptions. By the uniqueness part of
Thm.~\ref{thm:direct-aggregation-equation},
$\accdegrvec{\semdir}{\arggraph,d}=\accdegrvec{\semdir}{\arggraph',d}$.

\medskip\noindent Continuity: Matrix multiplication is continuous.

\medskip\noindent Linearity: Let $i$ be such that $a=a_i$.
Then $c_1=\sum_{1\leq j\leq n, j\not=i}\Pr^{G,d}_{ij}w_j$ and $c_2=\Pr^{G,d}_{ii}$.

\medskip\noindent Interchangeability: since $a_j$ and $a_k$ have the
same degree, the equation in
Thm.~\ref{thm:direct-aggregation-equation} is not affected by the
interchange.

\medskip\noindent Reverse impact: choose
$\accdegr{\semdir}{\arggraph,d}{b}=-1$ and proceed similar to the proof
of Impact.
\end{proof}

\bigskip
\begin{proof}{ of Thm.~\ref{thm:sigmoid-direct-aggr-properties}}
The proof parallels that of Thm.~\ref{thm:direct-aggr-properties},
using Equation~\ref{eq:sigmoid-fixpoint} instead of Thm.~\ref{thm:direct-aggregation-equation} and
noting that both $\sigma$ and $\sigma^{-1}$ are inverses of each other,
and each of them is strictly
monotone, continuous and commutes with permutation matrix multiplication.
For Neutrality and Impact use that $\sigma(0)=\frac{1}{2}$.
\end{proof}

In the proof of Thm.~\ref{thm:French-aggr-properties} below,
we essentially rely on the following fixed-point property, which
is Theorem~9 in \cite{DBLP:conf/ijcai/AmgoudB16}:
\begin{theorem}\label{thm:French-fixpoint}
$$\accdegr{\semrsig}{}{a}=w(a)+(1-w(a))\frac{\sum_{b\in \supporters{}{a}}\accdegr{\semrsig}{}{b}}{1+\sum_{b\in \supporters{}{a}}\accdegr{\semrsig}{}{b}}$$
or recast in matrix notation
$$\accdegrvec{\semrsig}{}=w+(I-Diag(w))f(G\accdegrvec{\semrsig}{})$$
where $f(x)=\frac{x}{1+x}$ is applied point-wise to a vector, and
$Diag$ uses a vector to fill the diagonal of a matrix, which is
otherwise zero.
\end{theorem}

\begin{proof}{ of Thm.~\ref{thm:French-aggr-properties}}
\medskip\noindent Anonymity, Independence, Equivalence: clear, since our
  definition (when restricted to graphs with supports relations only)
  is the same definition as in \cite{DBLP:conf/ijcai/AmgoudB16}.
  The same holds for Conservativity, which however is called Minimality
  in \cite{DBLP:conf/ijcai/AmgoudB16}.

\medskip\noindent Neutrality, Initial Monotony, Parent Monotony,
Reinforcement, Interchangeability, Neutralisation:
analogous to the proof of Thm.~\ref{thm:direct-aggr-properties}, where
now Thm.~\ref{thm:French-fixpoint} plays the role of
Thm.~\ref{thm:direct-aggregation-equation}, and noting that the
function $f$ is strictly monotonic. Neutralisation additionally
needs uniqueness for Thm.~\ref{thm:French-fixpoint}, which can
be shown using the convergence proof.

\medskip\noindent Directionality: 
Since there are no attack relations, we can ignore detractors.
The proof is analogous to the proof of Thm.~\ref{thm:direct-aggr-properties},
we still detail it here to indicate the necessary modifications.
We prove the following lemma (where we need to make $\arggraph$
explicit as parameter of $\fsemrsig$):
$$\forall y\in\backers{}{x}\cup\{x\},\ \fsemrsig_{\arggraph,i}(y)=\fsemrsig_{\arggraph',i}(y)$$
by induction over $i$. For $i=0$, both sides of the equation are $w(y)$.
Now let us prove the statement for $i+1$:  Since all for $y'\in\parent{}{y}$, we have
$y'\in\backers{}{x}\cup\{x\}$, by induction
hypothesis, $\fsemrsig_{\arggraph,i}(y')=\fsemrsig_{\arggraph',i}(y')$ for such $y'$, and hence
$\parent{}{y}\fsemrsig_{\arggraph,i}=\parent{}{y}\fsemrsig_{\arggraph',i}$.  
Now by the assumption, $y\not=
a_j$, and hence $\parent{}{y}=\parent{\arggraph^\prime}{y}$.  Thus
$\parent{}{y}\fsemrsig_{\arggraph',i}=\parent{\arggraph^\prime}{y}\fsemrsig_{\arggraph',i}$.
Hence altogether, $\fsemrsig_{\arggraph,i+1}(y)=
w(y)+(1-w(y))f(\parent{}{y}\fsemrsig_{\arggraph,i}(y))=
w(y)+(1-w(y))f(\parent{\arggraph^\prime}{y}\fsemrsig_{\arggraph^\prime,i}(y))=
\fsemrsig_{\arggraph',i}(y)$, and the lemma is proved.
Since the lemma applies in particular to $y=x$, we get
$\accdegr{\semrsig}{\arggraph}{x} =\lim_{i\to\infty} \fsemrsig_{\arggraph,i}(x) =\lim_{i\to\infty} \fsemrsig_{\arggraph',i}(y)
=\accdegr{\semrsig}{\arggraph^\prime}{x}$.

\medskip\noindent Impact: Modify the proof of Impact in
Thm.~\ref{thm:direct-aggr-properties} in the same way as we did for
Directionality above.

\medskip\noindent Causality: follows from Thm.~\ref{thm:Causality}.

\medskip\noindent Continuity: all the operations involved in the
definition of $\accdegrvec{\semrsig}{}$ are continuous.
\end{proof}

\end{document}